\documentclass[journal]{IEEEtran}
\usepackage{amsmath,amssymb,amsfonts}
\usepackage{cite}
\usepackage{algorithm}
\usepackage{algorithmic}
\usepackage{graphicx}
\usepackage{caption}
\usepackage{lipsum}
\usepackage{enumerate}
\usepackage{amsthm}
\usepackage{color,soul}
\usepackage{booktabs}
\usepackage{hyperref}
\hypersetup{
	bookmarksnumbered=true,
}
\newcommand{\Exp}{\mathop{\mathbb E}\displaylimits}

\newtheorem{lemma}{Lemma}
\newtheorem{theorem}{Theorem}

\newcommand\VECTOR{}  % 向量
\newcommand\SPACE{\mathcal}  % 集合

\ifCLASSOPTIONcompsoc
  \usepackage[caption=false,font=normalsize,labelfont=sf,textfont=sf]{subfig}
\else
  \usepackage[caption=false,font=footnotesize]{subfig}
\fi

\begin{document}

\title{Encoding Distributional Soft Actor-Critic for Autonomous Driving in Multi-lane Scenarios}

\author{Jingliang Duan, Yangang Ren, Fawang Zhang, Yang Guan, Dongjie Yu, Shengbo Eben Li,  Bo Cheng, Lin Zhao % <-this % stops a space
\thanks{%This study is supported by Beijing NSF with JQ18010, and NSF China with 51575293, and U20A20334. Special thanks should be given to Tencent for funding this study. 
Jingliang Duan and Yangang Ren contributed equally to this work. All correspondences should be sent to S. Li with email: lisb04@gmail.com. }% <-this % stops a space
\thanks{J. Duan, Y. Ren, Y. Guan, D. Yu, S. Li, and B. Cheng are with State Key Lab of Automotive Safety and Energy, School of Vehicle and Mobility, Tsinghua University, Beijing, 100084, China. They are also with Center for Intelligent Connected Vehicles and Transportation, Tsinghua University. {\tt\small Email: duanjl@nus.edu.sg, (ryg18, guany17, ydj20)@mails.tsinghua.edu.cn; (lishbo, chengbo)@tsinghua.edu.cn}.
}% <-this % stops a space
\thanks{J. Duan and L. Zhao are with the Department of Electrical and Computer Engineering, National University of Singapore, Singapore. {\tt\small Email: (duanjl,elezhli)@nus.edu.sg}.
}% <-this % stops a space
\thanks{F. Zhang is with the College of Engineering, China Agricultural University, Beijing, 100084, China. {\tt\small Email: fawang\_troy\_zhang@163.com}.
}% <-this % stops a space
}

\maketitle
\begin{abstract}
In this paper, we propose a new reinforcement learning (RL) algorithm, called encoding distributional soft actor-critic (E-DSAC), for decision-making in autonomous driving. Unlike existing RL-based decision-making methods, E-DSAC is suitable for situations where the number of surrounding vehicles is variable and eliminates the requirement for manually pre-designed sorting rules, resulting in higher policy performance and generality. We first develop an encoding distributional policy iteration (DPI) framework by embedding a permutation invariant module, 
which employs a feature neural network (NN) to encode the indicators of each vehicle, in the distributional RL framework. The proposed DPI framework is proved to exhibit important properties in terms of convergence and global optimality. Next, based on the developed encoding DPI framework, we propose the E-DSAC algorithm by adding the gradient-based update rule of the feature NN to the policy evaluation process of the DSAC algorithm. Then, the multi-lane driving task and the corresponding reward function are designed to verify the effectiveness of the proposed algorithm. Results show that the policy learned by E-DSAC can realize efficient, smooth, and relatively  safe  autonomous driving in the designed scenario. And the final policy performance learned by E-DSAC is about three times that of DSAC. Furthermore, its effectiveness has also been verified in real vehicle experiments.
\end{abstract}


\IEEEpeerreviewmaketitle

\section{Introduction}
Autonomous driving has tremendous potential to benefit traffic safety and efficiency as well as change radically future mobility and transportation. In the architecture of automated vehicles, decision-making is a key component to realize autonomy \cite{pendleton2017perception}.
Although the rule-based method has achieved fair performance in specific driving scenarios like DARPA \cite{buehler2009darpa}, it is non-trivial to design behavioral rules manually in a complex driving environment. Inspired by the success of reinforcement learning (RL) on games and robotic control \cite{mnih2015DQN,silver2017mastering}, RL-enabled methods have become a powerful learning framework in autonomous driving, capable of learning complicated policies for high dimensional environments.

As a pioneering work, Lillicap \emph{et al.} (2016) proposed a Deep Deterministic Policy Gradient (DDPG) algorithm, and employed it to realize lane-keeping on the TORCS platform with the simulated images as policy inputs \cite{lillicrap2015DDPG}. Inspired by this study, Wayve (2018) applied the DDPG algorithm to a real vehicle equipped with a monocular camera, achieving a fairly good lane-following effect on a 250m-long rural road \cite{wayve}. Wolf \emph{et al.} (2017) employed Deep Q-network (DQN) to learn a Q-value network for lane-keeping, which maps the simulated image to the expected return of 5 different steering angle values. At each instant, the ego vehicle selects and executes the steering angle with the highest Q-value \cite{wolf2017learning}. Besides, other RL algorithms, such as A3C \cite{perot2017end,jaritz2018end} , inverse RL \cite{zou2018inverse}, and SAC \cite{chen2019SACdriving}, were also utilized to learn policy, which makes decisions based on simulated images.

The common feature of the aforementioned works is to utilize the original sensor information such as images or point clouds as driving states, which are mapped to the corresponding control command through the learned policy. This is a typical end-to-end decision-making framework originated from DQN \cite{mnih2015DQN}, which maintains an integrated perception and decision-making module and eliminates the need for manually designed driving states. However, the data generated in the driving simulators during training usually suffers from a great discrepancy with real vehicle sensors. This characteristic probably leads to the lacking of robustness and limits its scalability to practical applications.  Additionally, to learn a policy based on the raw information, RL algorithms should simultaneously tackle two problems : (i) extract critical indicators like speed or relative distance between vehicles as the underlying state representation, and (ii) learning to maximize expected return. This places a heavy burden on the learning progress, especially in a complex driving environment. Therefore, the RL-based end-to-end decision-making methods are usually applied to simple driving tasks, such as lane keeping wherein no other surrounding vehicles are involved \cite{lillicrap2015DDPG,wayve,wolf2017learning,perot2017end,jaritz2018end, zou2018inverse}, and path tracking \cite{chen2019SACdriving}.

% As for more complex driving scenes concerning numerous surrounding traffic participants, as shown in the work of Intel laboratory \cite{dosovitskiy2017carla}, RL-enabled methods was not enough successful at avoiding collisions with cars and static objects compared with rule-based or supervised methods. 
Isele \emph{et al.} (2018) showed that compared with end-to-end decision-making that takes raw sensors data as inputs, real-valued representations, such as relative speed and distance from surrounding vehicles, can greatly improve the driving performance, due to the reduced state space being easier to learn and the meaningful indicators helping the system to generalize \cite{isele2018navigating}.
Under this scheme, Wang \emph{et al.} (2017) combined the recurrent neural network and the DQN algorithm to learn a policy for on-ramp merging, which outputs acceleration and steering angle based on the indicators of two surrounding vehicles on the target lane \cite{wang2017formulation}. Duan \emph{et al.} (2020) established a 26-dimensional state vector, consisting of the indicators of the four nearest surrounding vehicles, as well as the road and destination information \cite{duan2020hierarchical}. Based on this state description and a parallel RL algorithm, the ego vehicle successfully learned to execute multiple behaviors, such as car-following, lane-changing, and overtaking, on a simulated two-lane highway.
Guan \emph{et al.} (2020) developed a cooperative control for 8 connected automated vehicles at an unsignalized intersection \cite{GUAN}, where each vehicle is represented by its velocity and distance to the intersection center, thereby formulating a 16-dimensional state vector. The indicators of each vehicle are sorted according to a pre-designed order, and the Proximal Policy Optimization (PPO) algorithm \cite{schulman2017PPO} is adopted to obtain the final converged policy.

Although existing RL methods based on vector-based state representation have attained some elegant demonstrations in complex traffic scenarios, they still suffer two inevitable shortcomings resulting from the high dynamics of surrounding vehicles, namely, (1) dimension sensitivity and (2) permutation sensitivity. The dimension sensitivity means that RL-enabled algorithms relying on approximation functions, such as multi-layer neural network (NN), require that the input dimensions of policy and value functions are fixed. However, the state vector with a prior fixed dimension can poorly reflect surrounding vehicles because their number and relative positions within the perception range usually change dynamically during riding. Existing researches \cite{wang2017formulation,mirchevska2018high,duan2020hierarchical,wang2018reinforcement, wang2019continuous,GUAN} usually meet the fixed dimension requirement by removing the information of some vehicles or adding virtual vehicles. Obviously, the former will cause the loss of information, while the latter probably leads to a redundant representation. 

Permutation sensitivity means that for the same driving environment, different orders of surrounding vehicles in the state vector will formulate distinct state vectors,  resulting in quite different policy outputs. This unreasonable phenomenon will bring difficulties to policy learning. One intuitive approach is to manually design a permutation rule for all surrounding vehicles, for example, sorting vehicles in increasing order according to their relative distance from the ego \cite{duan2020hierarchical, GUAN}. However, this brings a burden of
manually designing permutation rules separately for different driving scenarios, and also introduces the discontinuity of the vectors, which is prone to reducing the policy performance.
Moreover, most RL-based decision-making algorithms have been developed and work well on sorts of simulators like CARLA \cite{dosovitskiy2017carla} and SUMO \cite{SUMO2018}, but their effectiveness on real vehicles has rarely been validated.

In this paper, we propose a new RL algorithm for autonomous driving  decision making, called encoding distributional soft actor-critic (E-DSAC), which can deal with the dimension sensitivity and permutation sensitivity problems mentioned above. The contributions and novelty of this paper are summarized as follows:
\begin{enumerate}
\item Inspired by the existing permutation invariant state representation method that introduces a feature NN to encode the indicators of each surrounding vehicle into an encoding vector \cite{duan2021fixeddimensional}, an encoding distributional policy iteration (DPI) framework is developed by embedding the encoding module in the distributional RL framework. The proposed DPI framework is proved to exhibit important properties in terms of convergence and global optimality. 
\item Based on the developed encoding DPI framework, we propose the encoding DSAC (E-DSAC) algorithm by adding the gradient-based update rule of feature NN to the policy evaluation process of the DSAC algorithm \cite{duan2021distributional}. Compared with existing RL-based decision-marking methods \cite{wang2017formulation,mirchevska2018high,duan2020hierarchical,wang2018reinforcement, wang2019continuous,GUAN}, E-DSAC  is suitable for situations where the number of surrounding vehicles is variable and eliminates the requirement for manually pre-designed sorting rules,  leading to higher policy performance and generality. 
\item The multi-lane driving task and the corresponding reward function are designed to verify the effectiveness of the proposed algorithm. Results show that the policy learned by E-DSAC can realize  efficient, smooth, and relatively  safe  autonomous driving in the designed scenario. And the final policy performance learned by E-DSAC is about three times that of DSAC. Furthermore, its effectiveness has also been verified in real vehicle experiments.
\end{enumerate}

The paper is organized as follows. In Section \ref{sec.preliminary}, we introduce the key notations and some preliminaries. Section \ref{sec:method} develops the encoding DPI framework and proposes the E-DSAC algorithm. Section \ref{sec:simulation} presents the simulation results in a four-lane highway driving scenario that show the efficacy of E-DSAC, and Section \ref{sec:real_veh_test} demonstrates the performance of E-DSAC in real vehicle experiments.  Finally, Section \ref{sec:conclusion} concludes this paper.

\section{Preliminaries}
\label{sec.preliminary}
In this section, we first introduce the basic principles of reinforcement learning (RL) and the distributional soft actor-critic (DSAC) algorithm. Then, we will show some key notations about the driving state representation.
\subsection{Basic Principles of RL}
\label{sec.notation}
We describe the autonomous driving process in a standard RL setting wherein an agent interacts with an environment in discrete time.
At current state $s_t$, the agent will take action $a_t$ according to policy $\pi$ and then the environment will return next state $s_{t+1}$ according to the environment model $p(s_{t+1}|s_t,a_t)$, i.e, $s_{t+1}  \sim p(s_{t+1}|s_t,a_t)$ and a scalar reward $r_t$. This process will repeat until the episode ends. In this paper, the policy is assumed to be stochastic, denoted as $\pi(a_t|s_t)$, which maps a given state to a probability distribution over actions. To ease the exposition, the current and next state-action pairs are also denoted as $(s,a)$ and $(s',a')$, respectively, and we will use $\rho_{\pi}$ to denote the state or state-action distribution induced by policy $\pi$.

The standard RL aims to learn a policy to maximize the expected accumulated return. This paper employs a more general entropy-augmented policy objective, which is widely used in soft RL methods \cite{Haarnoja2017Soft-Q,Haarnoja2018SAC,duan2021distributional}, to encourage exploration and further improve performance
\begin{equation}
\label{eq.policy_objective}
J_{\pi} = \Exp_{(s_{i \ge t},a_{i \ge t})\sim \rho_{\pi}}\Big[\sum^{\infty}_{i=t}\gamma^{i-t} [r_i+\alpha\mathcal{H}(\pi(\cdot|s_i))]\Big],
\end{equation}
where $\gamma \in (0, 1)$ is the discount factor, $\mathcal{H}(\pi(\cdot|s))=\Exp_{a\sim \pi}[-{\rm log} \pi(a|s)$] is the policy entropy, $\alpha$ is the coefficient that determines the relative importance of the entropy term against the reward. Under this scheme, we define the Q-value as
\begin{equation}
\label{eq.Q_definition}
\begin{aligned}
&Q^{\pi}(s_t,a_t)\\
&=\Exp_{\VECTOR{s_i,a_i\ge t}\sim\pi}\Big[r(s_t, a_t)+\sum^{\infty}_{i=t+1}\gamma^{i-t} [r(s_i,a_i)-\alpha {\rm log} \pi(a_i|s_i)]\Big]\\
&=r(s_t, a_t)+\gamma \Exp_{s'\sim p,a'\sim\pi}[Q^{\pi}(s',a') -\alpha {\rm log} \pi(a'|s') ].
\end{aligned}
\end{equation}
to indicate the return for choosing $a_t$ in state $s_t$ and thereafter following policy $\pi$. Then the objective \eqref{eq.policy_objective} of soft RL can be rewritten as 
\begin{equation}
\label{eq.policy_imp}
\pi_{\rm{new}}=\arg\max_{\pi} \Exp_{s\sim \rho_{\pi},a\sim \pi}\big[Q^{\pi}(s,a)-\alpha \log\pi(a|s)\big].
\end{equation}

\subsection{Distributional Soft Actor-Critic}
\label{sec.DSAC}
Under the framework of soft RL, we proposed a distributional soft actor-critic (DSAC) algorithm in 2021, which achieved SOTA performance in many continuous control tasks \cite{duan2021distributional}. Unlike mainstream RL algorithms that only learn the expected return, i.e., Q-value $Q(s,a)$, DSAC attempts to learn the distribution of return to reduce the overestimation of the value function, thereby improving policy performance. In DSAC, we define the state-action return as 
\begin{equation}
Z^{\pi}(s_t,a_t)=r(s_t,a_t)+\sum^{\infty}_{i=t+1}\gamma^{i-t} [r(s_i,a_i)-\alpha {\rm log} \pi(a_i|s_i)],
\end{equation}
which is a random variable due to the randomness in the state transition $p$ and policy $\pi$. We define $\mathcal{Z}^{\pi}(Z^{\pi}(s,a)|s,a): \mathcal{S}\times\mathcal{A}\rightarrow \mathcal{P}(Z^{\pi}(s,a))$ as a mapping from $(s,a)$ to a distribution over state-action returns, and call it the state-action return distribution or distributional value function. Obviously, the expectation of the return distribution is the Q-value in \eqref{eq.Q_definition}:
\begin{equation}
\label{eq.Q_and_distribution}
Q^{\pi}(s,a)=\Exp_{\substack{Z\sim \mathcal{Z}}}[Z^{\pi}(s,a)].
\end{equation}

The return distribution satisfies the following distributional Bellman operator
\begin{equation}
\label{eq.bellman}
\mathcal{T^{\pi}}Z^\pi(s,a) \overset{D}{=}r(s,a)+\gamma( Z^\pi(s',a')-\log\pi(a'|s')),
\end{equation}
where $s'\sim p$, $a'\sim \pi$, $A \overset{D}{=} B$ denotes that two random variables $A$ and $B$ have equal probability laws. To implement \eqref{eq.bellman}, the return distribution can be learned by
\begin{equation}
\label{eq.policy_eva}
\mathcal{Z}_{\rm{new}} =  \arg\min_{\mathcal{Z}}\Exp_{\rho_\pi}\big[D_{\rm KL}(\mathcal{T}^{\pi} \mathcal{Z}_{\rm{old}}(\cdot|s,a),\mathcal{Z}(\cdot|s,a))\big],
\end{equation}
where $\mathcal{T^{\pi}}Z(s,a)\sim\mathcal{T}^{\pi} \mathcal{Z}(\cdot|s,a)$ and $D_{\rm KL}$ represents the Kullback-Leibler (KL) divergence.

It has been proved that the distributional RL (DRL) framework that alternates between distributional policy evaluation based on \eqref{eq.policy_eva} and policy improvement based on \eqref{eq.policy_imp} will lead to the maximum entropy objective in \eqref{eq.policy_objective} \cite{duan2021distributional}.

\subsection{State Representation of Autonomous Driving}
In addition to the basic RL algorithm, to implement RL in autonomous driving, one essential thing is to design the state $s$ to reasonably describe the driving task based on the given observed information. Given a typical driving task, the observation $\mathcal{O}\in \overline{\mathcal{O}}$ from the driving environment should consist of two components: (a) the information set of surrounding vehicles $\mathcal{X}=\{x_1,x_2,...,x_M\}$, where $x_{i} \in \mathbb{R}^{d_1}$ is the real-valued indicator vector of the $i$th vehicle, and (b) the feature vector containing other indicators related to the ego vehicle and road geometry, denoted by $x_{\rm else} \in \mathbb{R}^{d_2}$. Thus, we can write $\mathcal{O}=\{\mathcal{X},x_{\rm else}\}$. Noted that in this setting, the observation $\mathcal{O}$ can be seen as the raw state of the autonomous driving problem, which is assumed to contain all the information required for policy learning. Therefore, the aforementioned equation in Section \ref{sec.notation} and \ref{sec.DSAC} also holds for $\mathcal{O}$. The reward function can be denoted as $r(\mathcal{O}, a, \mathcal{O'})$ in this case. 

During driving, the set size $M$ of $\mathcal{X}$, i.e., the number of surrounding vehicles within the perception range of the ego car, is constantly changing due to the dynamic nature of the traffic. Assuming that the range of the number of surrounding vehicles is $[1,N]$, the space of $\mathcal{X}$ can be denoted as $\overline{\mathcal{X}}=\{\mathcal{X}|\mathcal{X}=\{x_1,\cdots,x_{M}\},x_i\in\mathbb{R}^{d_1},i\le M,M\in[1,N]\cap\mathbb{N}\}$, i.e., $\mathcal{X}\in\overline{\mathcal{X}}$. Noted that the subscript $i$ of $x_{i}$ in $\mathcal{X}$ represents the ID of a certain vehicle. We denote the mapping from the observation $\mathcal{O}$ to the state vector $s$, which is taken as the policy and value inputs, as $U$, that is,
\begin{equation}
\label{eq.mapping_s}
s=U(\SPACE{O})=U(\SPACE{X},x_{\rm else}).
\end{equation}
One widely used mapping method, called fixed-permutation (FP) representation, is to directly concatenate the variables in $\mathcal{O}$, i.e.,
\begin{equation}
\label{eq.order_mapping}
s= U_{\rm FP}(\SPACE{O})=[x_{o(1)}^{\top},\dots,x_{o(M)}^{\top},x_{\rm else}^{\top}]^{\top},
\end{equation}
where $o$ denotes the pre-designed sorting rule, for instance, the surrounding vehicles are arranged in increasing order according to their relative distance from the ego.

As mentioned before, RL algorithms based on $U_{\rm FP}(\SPACE{O})$ suffer from two challenges: (1) dimension sensitivity and (2) permutation sensitivity. On one hand, from \eqref{eq.order_mapping}, dimension sensitivity means that the different number of surrounding vehicles would lead to different state dimension, i.e., ${\rm dim}(s)=Md_1+d_2$. Since the input dimension of policy and value functions should be fixed due to the structure limit of approximate functions, such as multi-layer neural networks (NNs), RL methods relying on \eqref{eq.order_mapping} can only consider a fixed number of surround vehicles. On the other hand, permutation sensitivity indicates that different permutations $o$ of $x_i$ correspond to different state vectors $s$, leading to different policy outputs. However, a reasonable driving decision should be permutation invariant to the order of objects in $\mathcal{X}$ because all possible permutations correspond to the same driving scenario.

\section{Encoding DSAC}
\label{sec:method}
To handle dimension sensitivity and
permutation sensitivity problems, this section proposes a permutation-invariant version of DSAC for self-driving decision-making, called encoding DSAC (E-DSAC), by embedding the permutation invariant state representation into DSAC.

\subsection{Permutation Invariant State Representation}
Firstly, we introduce a fixed-dimensional and permutation invariant state representation method called encoding sum and concatenation (ESC) \cite{duan2021fixeddimensional}. We employ a feature NN $h(\VECTOR{x};\phi )$ to encode the indicators of each surrounding vehicle $\VECTOR{x}\in\SPACE{X}$,
\begin{equation}
\VECTOR{x}_{\rm encode}=h(\VECTOR{x};\phi ),
\end{equation}
where $\VECTOR{x}_{\rm encode} \in \mathbb{R}^{d_3}$ is the corresponding encoding vector of each $\VECTOR{x} \in \mathbb{R}^{d_1}$ in the set $\SPACE{X}$, $\phi$ represents the parameters of feature NN.
Then, we obtain the representation vector $x_{\rm set}$ of all surrounding vehicles by summing the encoding vector of each surrounding vehicle
\begin{equation}
\label{eq04:PI_encode}
\VECTOR{x}_{\rm set}=\sum_{\VECTOR{x}\in\SPACE{X}}h(\VECTOR{x};\phi ).
\end{equation}
After concatenating with other indicators related to the ego vehicle and road geometry $x_{\rm else}$, we can obtain the final state representation $s$, 
\begin{equation}
\label{eq.pi_state}
    \VECTOR{s}=U_{\rm ESC}(\SPACE{O};\phi )=[\VECTOR{x}_{\rm set}^\top,\VECTOR{x}_{\rm else}^\top]^\top=\Big[\sum_{\VECTOR{x}\in\SPACE{X}}h^\top(\VECTOR{x};\phi ),\VECTOR{x}_{\rm else}^\top\Big]^\top.
\end{equation}

From \eqref{eq.pi_state}, it is clear that $\text{dim}(s)=\text{dim}(h_{\phi })+\text{dim}(x_{\rm else })=d_2+d_3$ for $\forall M \in [1,N]$. In other words, $U_{\rm ESC}(\SPACE{O};\phi )$ is fixed-dimensional. Furthermore, the summation operator in \eqref{eq04:PI_encode} is permutation invariant w.r.t. objects in $\SPACE{X}$. Thus, $U_{\rm ESC}(\SPACE{O};\phi )$ is a fixed-dimensional and permutation invariant state representation of observation $\mathcal{O}$. Besides, more importantly, the injectivity of $U_{\rm ESC}(\SPACE{O};\phi )$ can be guaranteed by carefully designing the architecture of the feature NN.
\begin{lemma}\label{lemma.encoding}
 (Injectivity of ESC\cite{duan2021fixeddimensional}). Let $\SPACE{O}=\{\SPACE{X},\VECTOR{x}_{\rm else}\}$, where $\VECTOR{x}_{\rm else}\in\mathbb{R}^{d_2}$ and $\mathcal{X}=\{x_1,x_2,...,x_M\}$. Denote the space of $\SPACE{X}$ as $\overline{\SPACE{X}}$, where $\overline{\mathcal{X}}=\{\mathcal{X}|\mathcal{X}=\{x_1,\cdots,x_{M}\},x_i\in[c_{\rm min},c_{\rm max}]^{d_1},i\le M,M\in[1,N]\cap\mathbb{N}\}$, in which $c_{\rm min}$ and $c_{\rm max}$ are the lower and upper bounds of all elements in $\forall\VECTOR{x_i}$, respectively. Noted that the size $M$ of the set $\mathcal{X}$ is variable. If the feature NN $h(\VECTOR{x};\VECTOR{\phi}):\mathbb{R}^{d_1}\rightarrow\mathbb{R}^{d_3}$ is over-parameterized (i.e., the number of hidden neurons is sufficiently large) with a linear output layer, and its output dimension $d_3\ge Nd_1+1$, there always $\exists \phi^{\dagger}$ such that the mapping $U_{\rm ESC}(\mathcal{O};\phi^{\dagger}): \overline{\mathcal{X}}\times \mathbb{R}^{d_2}\rightarrow \mathbb{R}^{d_3+d_2}$ in \eqref{eq.pi_state} is injective.
\end{lemma}

% \begin{figure}[!htb]
%     \captionsetup{justification=raggedright, 
%                   singlelinecheck=false, font=small}
%     \centering{\includegraphics[width=0.5\textwidth]{figure/ESC-policy.pdf}}
%     \caption{Policy based on ESC state representation.}
%     \label{fig.PI_module}
% \end{figure}

\subsection{Encoding Distributional Policy Iteration}
We assume that the random returns $Z(s,a)$ and action $a$ obey the Gaussian distribution. Therefore,  both the state-action return distribution and policy functions are modeled as Gaussian with mean and covariance given by NN, denoted as  $\mathcal{Z}(\cdot|s,a;\theta)$ and $\pi_{\omega}(\cdot|s;\omega)$, where $\theta$ and $\omega$ are parameters. For ease of presentation, we will also denote $\mathcal{Z}(\cdot|s,a;\theta)$, $\pi(\cdot|s;\omega)$, and $h(x;\phi)$ as $\mathcal{Z}_{\theta}(\cdot|s,a)$, $\pi_{\omega}(\cdot|s)$, and $h_{\phi}(x)$, respectively, when it is clear from context.
By taking $U_{\rm ESC}(\SPACE{O};\phi )$ as the input of $\pi_{\omega }$, the policy function can be expressed as 
\begin{equation}
\begin{aligned}
\pi(U_{\rm ESC}(\SPACE{O};\phi );\omega )=\pi(\sum_{\VECTOR{x}\in\SPACE{X}}h(\VECTOR{x};\phi ),\VECTOR{x}_{\rm else};\omega ).
\end{aligned}
\end{equation}
Similarly, the return distribution $\mathcal{Z}(s,a)$ can be formalized as
\begin{equation}
\begin{aligned}
\mathcal{Z}(U_{\rm ESC}(\SPACE{O};\phi );\theta )=\mathcal{Z}(\sum_{\VECTOR{x}\in\SPACE{X}}h(\VECTOR{x};\phi ),\VECTOR{x}_{\rm else};\theta ).
\end{aligned}
\end{equation}
It is clear that both $\pi_\omega(U_{\rm ESC}(\SPACE{O};\phi ))$ and $\mathcal{Z}_{\theta}(U_{\rm ESC}(\SPACE{O};\phi ))$ are permutation invariant to objects in $\mathcal{X}$. This enables us to use DSAC to learn a permutation invariant policy for autonomous driving considering variable surrounding vehicles.

By embedding the ESC state representation into the DRL framework, the objective of distributional policy evaluation in \eqref{eq.policy_eva} can be rewritten as
\begin{equation}
\label{eq.encoding_policy_eva_objec}
\begin{aligned}
&J_{\mathcal{Z}}(\theta,\phi)=\Exp_{\rho_\pi}\big[D_{\rm KL}(\mathcal{T}^{\pi}\mathcal{Z}_{\theta_{\rm old}}(\cdot|s,a)\big|_{\VECTOR{s}=U_{\rm ESC}(\SPACE{O};\VECTOR{\phi_{\rm old}})},\\
&\qquad\qquad\qquad\qquad\qquad\qquad\mathcal{Z}_{\theta}(\cdot|s,a)\big|_{\VECTOR{s}=U_{\rm ESC}(\SPACE{O};\phi )})\big].
\end{aligned}
\end{equation}
In this case, $\rho_{\pi}$ represents the observation-action distribution induced by policy $\pi$. Noted that \eqref{eq.encoding_policy_eva_objec} is a joint optimization objective for both feature and value networks. Then, from \eqref{eq.policy_imp}, the policy network will be updated by maximizing the following objective
\begin{equation}
\label{eq:encoding_policy_imp}
J_{\pi}(\omega )=\Exp_{\mathcal{O}\sim \rho_\pi,\VECTOR{a}\sim\pi_{\omega }}[Q(\VECTOR{s},\VECTOR{a};\theta )-\alpha\log\pi(\VECTOR{a}|\VECTOR{s};\omega )\big|_{\VECTOR{s}=U_{\rm ESC}(\SPACE{O};\phi )}].
\end{equation}

Based on \eqref{eq.encoding_policy_eva_objec} and \eqref{eq:encoding_policy_imp}, we can derive the encoding DRL framework shown in Algorithm \ref{alg:EDPI}.
\begin{algorithm}[!htb]
\caption{Encoding DPI Framework}
\label{alg:EDPI}
\begin{algorithmic}
\STATE Initialize parameters $\theta $, $\omega $, $\phi $ and entropy coefficient $\alpha$
\REPEAT
\STATE 1. Encoding Distributional Policy Evaluation 

\setlength{\leftskip}{2em}
\STATE Estimate $V_{\theta}$ and $h_\phi$ using policy $\pi_\omega$
\REPEAT
\STATE \begin{equation}
\label{eq.encoding_PE}
\{\theta,\phi\}
\leftarrow \arg\min_{\theta,\phi}J_{\mathcal{Z}}(\theta,\phi)
\end{equation}
\UNTIL Convergence 

\setlength{\leftskip}{0em}
\STATE 2. Encoding Policy Improvement
\begin{equation}
\label{eq.encoding_PI}
\omega \leftarrow \arg\max_{\omega}J_{\pi}(\omega)
\end{equation}
\UNTIL Convergence 
\end{algorithmic}
\end{algorithm}
Next, we will prove that the proposed encoding DPI framework would lead to policy improvement with respect to the maximum entropy objective.
\begin{lemma} 
\label{lemma.UAT}
(Universal Approximation Theorem \cite{Hornik1990Universal}). For any continuous function $F(x):\mathbb{R}^n\rightarrow\mathbb{R}^d$ on a compact set $\Omega$, there exists an over-parameterized NN, which uniformly approximates $F(x)$ and its gradient to within arbitrarily small error $\epsilon \in \mathbb{R}_{+}$ on $\Omega$.
\end{lemma}

\begin{lemma}\label{lemma.edpe}
(Encoding Distributional Policy Evaluation). Suppose both feature NN $h(\VECTOR{x};\VECTOR{\phi}):\mathbb{R}^{d_1}\rightarrow\mathbb{R}^{d_3}$ and value NN $\mathcal{Z}(\cdot|s,a; \theta)$ are over-parameterized, with $d_3\ge Nd_1+1$. Let
\begin{equation}
\nonumber
\begin{aligned}
&\{\theta^{i+1},\phi^{i+1}\}=\arg\min_{\{\theta,\phi\}}\Exp_{\rho_\pi}\big[D_{\rm KL}(\mathcal{T}^{\pi}\mathcal{Z}_{\theta_i}(\cdot|s,a)\big|_{\VECTOR{s}=U_{\rm ESC}(\SPACE{O};\VECTOR{\phi_i})},\\
&\qquad\qquad\qquad\qquad\qquad\qquad\mathcal{Z}_{\theta}(\cdot|s,a)\big|_{\VECTOR{s}=U_{\rm ESC}(\SPACE{O};\phi )})\big],
\end{aligned}
\end{equation}
and denote the state-action return distribution as $\mathcal{Z}^i(\cdot|\mathcal{O},a)=\mathcal{Z}(\cdot|U_{\rm ESC}(\mathcal{O};\phi^{i}),a;{\theta^{i}})$, which maps the observation-action pair $(\mathcal{O},a)$ to a distribution over random state-action returns $Z^i(\mathcal{O},a)$, i.e., $Z^i(\mathcal{O},a)\sim\mathcal{Z}^i(\cdot|\mathcal{O},a)$.  Consider the distributional bellman backup operator $\mathcal{T}^{\pi}$ in \eqref{eq.bellman} and define $\mathcal{T}^{\pi}\mathcal{Z}^i(\cdot|\mathcal{O},a)$ as the distribution of $\mathcal{T}^{\pi}Z^i(\mathcal{O},a)$, i.e., $\mathcal{T}^{\pi}Z^i(\mathcal{O},a)\sim\mathcal{T}^{\pi}\mathcal{Z}^i(\cdot|\mathcal{O},a)$. Then, given any policy $\pi$, the sequence $\mathcal{Z}^{i}$ will converge to $\mathcal{Z}^{\pi}$ as $i\rightarrow \infty$.
\end{lemma}
\begin{proof}
Let $\overline{Z}$ denote the space of soft return function $Z$. From Lemma \ref{lemma.encoding}, there exists $\phi_{i+1}$ such that $U_{\rm ESC}(\mathcal{O};\phi^{i+1})$ is an injective mapping. Then, based on Lemma \ref{lemma.UAT}, given any parameters $\{\theta^{i},\phi^{i}\}$, there exists $\theta^{i+1}$ such that 
\begin{equation}
\nonumber
\begin{aligned}
&D_{\rm KL}(\mathcal{T}^{\pi}\mathcal{Z}_{\theta^{i}}(\cdot|s,a)\big|_{U_{\rm ESC}(\mathcal{O};\phi^{i})},\\
&\qquad\qquad\qquad\qquad\mathcal{Z}_{\theta^{i+1}}(\cdot|s,a)\big|_{U_{\rm ESC}(\mathcal{O};\phi^{i+1})}) = 0,  
\end{aligned}
\end{equation}
which is equivalent to
\begin{equation}
\nonumber
\mathcal{Z}_{\theta^{i+1}}(\cdot|U_{\rm ESC}(\mathcal{O};\phi^{i+1}),a)=\mathcal{T}^{\pi}\mathcal{Z}_{\theta^{i}}(\cdot|U_{\rm ESC}(\mathcal{O};\phi^{i}),a).
\end{equation}
Then, we can directly apply the standard convergence results for distributional policy evaluation (Lemma 1 of \cite{duan2021distributional}), that is, $\mathcal{T}^{\pi}: \overline{Z}\rightarrow\overline{Z}$ is a $\gamma$-contraction in terms of some measure. Therefore, $\mathcal{T}^{\pi}_{\mathcal{D}}$ has a unique fixed point, which is $Z^{\pi}$, and the sequence $Z^{i}$ will converge to it as $i\rightarrow \infty$, i.e., $\mathcal{Z}^{i}$ will converge to $\mathcal{Z}^{\pi}$ as $i\rightarrow \infty$.

\end{proof}

\begin{lemma}\label{lemma.epi}
(Encoding Policy Improvement) Suppose 
\begin{equation}
\label{eq.lemma4_assump}
 \mathcal{Z}(\cdot|U_{\rm ESC}(\mathcal{O};\phi_{\rm old}),a;\theta_{\rm old})=\mathcal{Z}^{\pi_{\omega_{\rm old}}}(\cdot|\mathcal{O},a),
\end{equation}
and the policy NN $\pi(\cdot|s)$ is over-parameterized.
 Let 
\begin{equation}
\label{eq.lemma4_updaterule}
\begin{aligned}
&\omega_{\rm new}=\arg\max_{\omega}J_{\pi}(\omega),
\end{aligned}
\end{equation}
where
\begin{equation}
\nonumber
J_{\pi}(\omega)=\Exp_{\substack{\mathcal{O}\sim \rho_\pi,\\\VECTOR{a}\sim\pi_{\omega }}}[Q(\VECTOR{s},\VECTOR{a};\theta_{\rm old} )-\alpha\log\pi_{\omega}(\VECTOR{a}|\VECTOR{s} )\big|_{\VECTOR{s}=U_{\rm ESC}(\SPACE{O};\phi_{\rm old})}].
\end{equation}
Then $Q^{\pi_{\omega_{\rm new}}}(\mathcal{O},a) \ge Q^{\pi_{\omega_{\rm old}}}(\mathcal{O},a)$ for $\forall(\mathcal{O},a) \in \overline{\mathcal{O}}\times\mathcal{A}$.
\begin{proof}
Firstly, let
\begin{equation}
\label{eq.lemma4_another_policy}
\pi_{\rm new}(\cdot|\mathcal{O})=\arg\max_{\pi}\Exp_{\VECTOR{a}\sim\pi }[Q^{\pi_{\omega_{\rm old}}}(\mathcal{O},\VECTOR{a})-
\alpha\log\pi(\VECTOR{a}|\mathcal{O})],
\end{equation}
for $\forall \mathcal{O}\in \overline{\mathcal{O}}$.

Given an arbitrary problem, the return distributions of two different observations $\mathcal{O}\neq \mathcal{O'} \in \overline{\mathcal{O}}$  fall into the following two cases:
\begin{enumerate}[\quad(1)] 
\item $ \exists a\in\mathcal{A}$ such that $\mathcal{Z}^{\pi_{\omega_{\rm old}}}(\cdot|\mathcal{O},a)\neq\mathcal{Z}^{\pi_{\omega_{\rm old}}}(\cdot|\mathcal{O'},a)$; 
\item $\mathcal{Z}^{\pi_{\omega_{\rm old}}}(\cdot|\mathcal{O},a)=\mathcal{Z}^{\pi_{\omega_{\rm old}}}(\cdot|\mathcal{O'},a)$ for $ \forall a\in\mathcal{A}$. 
\end{enumerate}
For case 1, to ensure \eqref{eq.lemma4_assump} holds, it must follow that 
\begin{equation}
\mathcal{O}\neq \mathcal{O'} \rightarrow
U_{\rm ESC}(\mathcal{O};\phi_{\rm old}) \neq U_{\rm ESC}(\mathcal{O'};\phi_{\rm old}), 
\end{equation}
that is, $U_{\rm ESC}(\mathcal{O};\phi_{\rm old})$ is an injective function. For case 2, although there is a possibility that $U_{\rm ESC}(\mathcal{O};\phi_{\rm old}) = U_{\rm ESC}(\mathcal{O'};\phi_{\rm old})$, we always have that $\pi_{\rm new}(\cdot|\mathcal{O})=\pi_{\rm new}(\cdot|\mathcal{O'})$ since $Q^{\pi_{\omega_{\rm old}}}(\mathcal{O},\VECTOR{a})=Q^{\pi_{\omega_{\rm old}}}(\mathcal{O'},\VECTOR{a})$. Therefore, one has
\begin{equation}
\begin{aligned}
\pi_{\rm new}(\cdot|\mathcal{O})\neq\pi_{\rm new}(\cdot|\mathcal{O'})\rightarrow U_{\rm ESC}(\mathcal{O};\phi_{\rm old}) \neq U_{\rm ESC}(\mathcal{O'};\phi_{\rm old}).
\end{aligned}
\end{equation}

Furthermore, from Lemma \ref{lemma.UAT}, there exist $\omega^{\dagger}$ such that 
\begin{equation}
\label{eq.policy_equal}
\pi(\cdot|s;\omega^{\dagger})\big|_{\VECTOR{s}=U_{\rm ESC}(\SPACE{O};\VECTOR{\phi_{\rm old}})}=\pi_{\rm new}(\cdot|\mathcal{O}),\quad \forall \mathcal{O}\in \overline{\mathcal{O}}.
\end{equation}
Based on \eqref{eq.lemma4_updaterule}, one has
\begin{equation}
\label{eq.policy_inequality_1}
\begin{aligned}
J_{\pi}(\omega^{\dagger})\le \max_{\omega} J_{\pi}(\omega) = J_{\pi}(\omega_{\rm new}).
\end{aligned}
\end{equation}
From \eqref{eq.lemma4_assump} and \eqref{eq.Q_and_distribution}, it is clear that
\begin{equation}
\label{eq:Q_relation2}
Q^{\pi_{\omega_{\rm old}}}(\mathcal{O}, a)=Q(s, a;{\theta_{\rm old}})\big|_{s=U_{\rm ESC}(\mathcal{O};\phi_{\rm old})},\forall \{\mathcal{O},a\}\in \overline{\mathcal{O}}\times\mathcal{A}.
\end{equation}
Then, according to \eqref{eq.lemma4_another_policy},
\begin{equation}
\label{eq.policy_inequality_2}
\begin{aligned}
&J_{\pi}(\omega_{\rm new})\\
&=\Exp_{\substack{\mathcal{O}\sim \rho_{\pi},\\\VECTOR{a}\sim\pi_{\omega_{\rm new}}}}[Q^{\pi_{\omega_{\rm old}}}(\mathcal{O}, a)-\alpha\log\pi_{\omega_{\rm new}}(\VECTOR{a}|\VECTOR{s} )\big|_{\VECTOR{s}=U_{\rm ESC}(\SPACE{O};\phi_{\rm old})}]\\
&\le \Exp_{\substack{\mathcal{O}\sim \rho_\pi,\\a\sim \pi_{{\rm new}}}}[Q^{\pi_{\omega_{\rm old}}}(\mathcal{O},a)-\alpha\log\pi_{{\rm new}}(a|\mathcal{O})].
\end{aligned}
\end{equation}
Therefore, combining \eqref{eq.policy_equal}, \eqref{eq.policy_inequality_1} and \eqref{eq.policy_inequality_2}, it follows that 
\begin{equation}
\label{eq.lemma4_policy_equality}
J_{\pi}(\omega_{\rm new})=\Exp_{\substack{\mathcal{O}\sim\rho_\pi,\\a\sim \pi_{{\rm new}}}}[Q^{\pi_{\omega_{\rm old}}}(\mathcal{O},a)-\alpha\log\pi_{{\rm new}}(a|\mathcal{O})].
\end{equation}

From \eqref{eq.lemma4_another_policy}, for $\forall \mathcal{O} \in \overline{\mathcal{O}}$, one has 
\begin{equation}
\label{eq.lemma4_value_imp}
\begin{aligned}
&\Exp_{a\sim \pi_{{\rm new}}}[Q^{\pi_{\omega_{\rm old}}}(\mathcal{O},a)-\alpha\log\pi_{{\rm new}}(a|\mathcal{O})]\ge \\
& \qquad \Exp_{a\sim \pi_{\omega_{\rm old}}}[Q^{\pi_{\omega_{\rm old}}}(\mathcal{O},a)-\alpha\log\pi_{\omega_{\rm old}}(\VECTOR{a}|\VECTOR{s} )\big|_{\VECTOR{s}=U_{\rm ESC}(\SPACE{O};\phi_{\rm old})}].
\end{aligned}
\end{equation} 
Based on \eqref{eq.lemma4_policy_equality}, by replacing $ \pi_{{\rm new}}$ with $ \pi_{\omega_{\rm new}}$, \eqref{eq.lemma4_value_imp} can be rewritten as
\begin{equation}
\begin{aligned}
&\Exp_{a\sim \pi_{\omega_{\rm new}}}[Q^{\pi_{\omega_{\rm old}}}(\mathcal{O},a)-\alpha\log\pi_{\omega_{\rm new}}(a|s)\big|_{\VECTOR{s}=U_{\rm ESC}(\SPACE{O};\VECTOR{\phi_{\rm old}})}]\ge \\
& \qquad  \Exp_{a\sim \pi_{\omega_{\rm old}}}[Q^{\pi_{\omega_{\rm old}}}(\mathcal{O},a)-\alpha\log\pi_{\omega_{\rm old}}(\VECTOR{a}|\VECTOR{s} )\big|_{\VECTOR{s}=U_{\rm ESC}(\SPACE{O};\phi_{\rm old})}].
\end{aligned}
\end{equation} 

Next, from \eqref{eq.Q_definition}, it follows that
\begin{equation}
\nonumber
\begin{aligned}
&Q^{\pi_{\omega_{\rm old}}}(\mathcal{O}, a) \\
&= r+\gamma\Exp_{\mathcal{O}',a'\sim \pi_{\omega_{\rm old}}}[Q^{\pi_{\omega_{\rm old}}}(\mathcal{O}',a')-\alpha\log\pi_{\omega_{\rm old}}(a'|s')]\\
&\le r+\gamma\Exp_{\mathcal{O}',a'\sim \pi_{\omega_{\rm new}}}[Q^{\pi_{\omega_{\rm old}}}(\mathcal{O}',a')-\alpha\log\pi_{\omega_{\rm new}}(a'|s')]\\
&\vdots\\
&\le Q^{\pi_{\omega_{\rm new}}}(\mathcal{O},a),  \quad \forall(\mathcal{O},a)\in\overline{\mathcal{O}}\times\mathcal{A},
\end{aligned}
\end{equation}
where $s'=U_{\rm ESC}(\SPACE{O'};\VECTOR{\phi_{\rm old}})$ and the last step is derived by repeatedly expanding $Q^{\pi_{\omega_{\rm old}}}$ on the right-hand side by applying \eqref{eq.Q_definition}.
\end{proof}
\end{lemma}

\begin{theorem}\label{theorem.dspi}
(Encoding Distributional Policy Iteration). The encoding distributional policy iteration, which alternates between encoding distributional policy evaluation and encoding policy improvement, can converge to a policy $\pi^*$ such that $Q^{\pi^*}(\mathcal{O}, a)\ge Q^{\pi}(\mathcal{O}, a)$ for $\forall\pi$ and $\forall (\mathcal{O}, a)\in\overline{\mathcal{O}}\times\mathcal{A}$, assuming that $|\mathcal{A}|<\infty$ and reward is bounded.
\end{theorem}
\begin{proof}
Let $\pi_{\omega_k}$ denote the policy at iteration $k$. For $\forall \pi_{\omega_k}$, we can always find $\{\theta_k,\phi_k\}$ such that $\mathcal{Z}(\cdot|U_{\rm ESC}(\mathcal{O};\phi_{k}),a;{\theta_{k}})=\mathcal{Z}^{\pi_{\omega_k}}(\cdot|\mathcal{O},a)$ through distributional soft policy evaluation process following from Lemma \ref{lemma.edpe}. Therefore, we can obtain $Q^{\pi_{\omega_k}}(\mathcal{O},a)$ according to \eqref{eq.Q_and_distribution}. By Lemma \ref{lemma.epi}, the sequence $Q^{\pi_{\omega_k}}(\mathcal{O},a)$ is monotonically increasing for $\forall(\mathcal{O},a)\in\overline{\mathcal{O}}\times\mathcal{A}$. Since $Q^{\pi}$ is bounded everywhere for $\forall\pi$ (both the reward and policy entropy are bounded), the policy sequence  $\pi_{\omega_k}$ converges to some $\pi^{\dagger}$ as $k\rightarrow\infty$. At convergence, it must follow that
\begin{equation}
\begin{aligned}
&\Exp_{a\sim \pi^{\dagger}}[Q^{\pi^{\dagger}}(s,a)-\alpha\log\pi^{\dagger}(a|s)]\ge \\
&\qquad \quad \Exp_{a\sim \pi}[Q^{\pi^{\dagger}}(s,a)-\alpha\log\pi(a|s)],\quad \forall \pi, \forall s \in \mathcal{S}.
\end{aligned}
\end{equation}
Using the same iterative argument as in Lemma \ref{lemma.epi}, we have
\begin{equation}
\nonumber
 Q^{\pi^{\dagger}}(s, a) \ge Q^{\pi}(s, a), \quad \forall \pi, \forall(\mathcal{O},a)\in\overline{\mathcal{O}}\times\mathcal{A}.
\end{equation}
Hence $\pi^{\dagger}$ is optimal, i.e., $\pi^{\dagger}=\pi^*$.
\end{proof}

\subsection{Algorithm}
For practical applications, directly solving \eqref{eq.encoding_PE} and \eqref{eq.encoding_PI} may be intractable due to the high-dimensional and nonlinear characteristics of NNs. In this case, the gradient method is an effective manner to find the nearly optimal solutions iteratively. 

\subsubsection{Encoding Distributional Policy Evaluation}
To derive the gradient of value and feature functions, we first rewrite their objective \eqref{eq.encoding_policy_eva_objec} as 
\begin{equation}
\nonumber
\begin{aligned}
&J_{\mathcal{Z}}(\theta,\phi)\\
&=  \mathbb{E}_{(\SPACE{O},a)\sim\mathcal{B}}\Big[D_{\rm{KL}}(\mathcal{T}^{\pi_{\omega'}}_{\mathcal{D}}\mathcal{Z}_{\theta'}(\cdot|s,a),\mathcal{Z}_{\theta}(\cdot|s,a))\Big]\\
&=  \mathbb{E}_{(\SPACE{O},a)\sim\mathcal{B}}\Big[\sum_{\mathcal{T}^{\pi_{\omega'}}_{\mathcal{D}}Z(s,a)}\mathcal{T}^{\pi_{\omega'}}_{\mathcal{D}}\mathcal{Z}_{\theta'}(\mathcal{T}^{\pi_{\omega'}}_{\mathcal{D}}Z(s,a)|s,a)\\
& \qquad\qquad\qquad\qquad\qquad \log\frac{\mathcal{T}^{\pi_{\omega'}}_{\mathcal{D}}\mathcal{Z}_{\theta'}(\mathcal{T}^{\pi_{\omega'}}_{\mathcal{D}}Z(s,a)|s,a)}{\mathcal{Z}_{\theta}(\mathcal{T}^{\pi_{\omega'}}_{\mathcal{D}}Z(s,a)|s,a)}\Big]\\
&=  -\mathbb{E}_{(\SPACE{O},a)\sim\mathcal{B}}\Big[\mathbb{E}_{\mathcal{T}^{\pi_{\omega'}}_{\mathcal{D}}Z(s,a)\sim\mathcal{T}^{\pi_{\omega'}}_{\mathcal{D}}\mathcal{Z}_{\theta'}(\cdot|s,a)}\\
&\qquad \qquad \qquad \qquad \qquad \log\mathcal{Z}_{\theta}(\mathcal{T}^{\pi_{\omega'}}_{\mathcal{D}}Z(s,a)|s,a)\Big]+c\\
&=  -\mathbb{E}_{(\SPACE{O},a)\sim\mathcal{B}}\Big[\Exp_{\substack{(r,\mathcal{O}')\sim \mathcal{B},a'\sim\pi_{\omega'},\\ Z(s',a')\sim\mathcal{Z}_{\theta'}(\cdot|s',a')}}\\
&\qquad \qquad \qquad \qquad \qquad\log\mathcal{Z}_{\theta}(\mathcal{T}^{\pi_{\omega'}}_{\mathcal{D}}Z(s,a)|s,a)\Big]+c\\
&= -\Exp_{\substack{(\mathcal{O},a,r,\mathcal{O}')\sim\mathcal{B},a'\sim\pi_{\omega'},\\Z(s',a')\sim\mathcal{Z}_{\theta'}(\cdot|s',a')}}\Big[\log\mathcal{Z}_{\theta}(\mathcal{T}^{\pi_{\omega'}}_{\mathcal{D}}Z(s,a)|s,a)\Big]+c,
\end{aligned}
\end{equation}
where $\VECTOR{s}=U_{\rm ESC}(\SPACE{O};\phi )$, $\VECTOR{s}'=U_{\rm ESC}(\SPACE{O}';\phi ')$, $\mathcal{B}$ is a replay buffer of previously sampled experience, $\theta'$, $\phi'$ and $\omega'$ are parameters of target return distribution, feature and policy NNs, which are used to stabilize the learning process and evaluate the target.

Then, the corresponding gradient of the value network can be derived as
\begin{equation}
\label{eq05:value_gradient}
\begin{aligned}
&\nabla_{\theta }J_{\SPACE{Z}}(\theta ,\phi )=\\
& -\Exp_{\substack{(\mathcal{O},a,r,\mathcal{O}')\sim\mathcal{B},\\
a'\sim\pi_{\omega'},\\Z(s',a')\sim\mathcal{Z}_{\theta'}}}\Big[\nabla_\theta\log\mathcal{Z}_{\theta}(\mathcal{T}^{\pi_{\omega'}}_{\mathcal{D}}Z(s,a)|s,a)\big|_{\substack{\VECTOR{s}=U_{\rm ESC}(\SPACE{O};\phi ),\\\VECTOR{s}'=U_{\rm ESC}(\SPACE{O}';\phi ')}}\Big],
\end{aligned}
\end{equation}
and the gradient w.r.t the feature network can be calculated as
\begin{equation}
\label{eq05:encode_gradient}
\begin{aligned}
&\nabla_{\phi }J_{\SPACE{Z}}(\theta ,\phi )\\
&= -\Exp_{\substack{(\mathcal{O},a,r,\mathcal{O}')\sim\mathcal{B},\\
a'\sim\pi_{\omega'},\\Z(s',a')\sim\mathcal{Z}_{\theta'}}}\Big[
\frac{\partial\log\SPACE{Z}(\SPACE{T}^{\pi_{\omega '}}_{\SPACE{D}}Z(\VECTOR{s},\VECTOR{a})|\VECTOR{s},\VECTOR{a};\theta )}{\partial \VECTOR{x}_{\rm set}}\times\\
&\qquad\qquad \qquad \qquad\qquad \qquad \sum_{\VECTOR{x}\in \SPACE{X}}\nabla_{\phi }h(\VECTOR{x};\phi ) \Big|_{\substack{\VECTOR{s}=U_{\rm ESC}(\SPACE{O};\phi ),\\\VECTOR{s}'=U_{\rm ESC}(\SPACE{O}';\phi ')}}\Big].
\end{aligned}
\end{equation}

\subsubsection{Encoding Policy Improvement}
Since $\mathcal{Z}_{\theta}$ is assumed to be a Gaussian model, it can be expressed as $\mathcal{Z}_{\theta}(\cdot|s,a)=\mathcal{N}(Q_{\theta}(s,a),\sigma_{\theta}(s,a)^2)$, where $Q_{\theta}(s,a)$ and $\sigma_{\theta}(s,a)$ are the outputs of value network. To obtain the gradient of the policy NN, we first need to reparameterize the stochastic policy  $\pi_\omega(a|s)$ using the following deterministic form 
\begin{equation}
\nonumber
\label{eq03:repara_policy}
\VECTOR{a}=\check{\pi}_\omega(\VECTOR{s},\xi),
\end{equation}
where $\xi$ is an auxiliary random variable and $\check\pi$ is the reparameterized policy. In particular, since $\pi_{\omega}(\cdot|s)$ is assumed to be a Gaussian in this paper, $\pi_{\omega}(s,\xi)$ can be formulated as 
\begin{equation}
\nonumber
\check{\pi}_\omega(s,\xi)=a_{\rm{mean}}+\xi \odot a_{\rm{std}},
\end{equation}
where $a_{\rm{mean}}\in \mathbb{R}^{{\rm{dim}}(\mathcal{A})}$ and $a_{\rm{std}}\in \mathbb{R}^{{\rm{dim}}(\mathcal{A})}$ are the mean and standard deviation of $\pi_{\omega}(\cdot|s)$, $\odot$ represents the Hadamard product and ${\xi} \sim \mathcal{N}(0,\bf{I}_{{\rm{dim}}(\mathcal{A})})$.
Then the policy update gradients can be approximated with 
\begin{equation}
\label{eq04:policy_gradient_based_on_Q_PI}
\begin{aligned}
&\nabla_{\omega }J_{\pi}(\omega )=\mathbb{E}_{\mathcal{O}\sim\SPACE{\beta},\xi}\Big[-\alpha\nabla_{\omega }\log\pi(\VECTOR{a}|\VECTOR{s};\omega )+\big(\nabla_{\VECTOR{a}}Q(\VECTOR{s},\VECTOR{a};\theta )
\\&\qquad\qquad-\alpha\nabla_{\VECTOR{a}}\log\pi(\VECTOR{a}|\VECTOR{s};\omega )\big)\nabla_{\omega }\check{\pi}(\VECTOR{s},\xi;\omega)\big|_{\substack{\VECTOR{s}=U_{\rm ESC}(\SPACE{O};\phi ),\\\VECTOR{a}=\check{\pi}_{\omega}(\VECTOR{s},\xi)}}\Big].
\end{aligned}
\end{equation}

\subsubsection{Pseudocode}
As for the target NNs, we adopt a slow-moving update rate to stabilize the learning process, that is,
\begin{equation}
\label{eq.target_update}
\begin{aligned}
y' \leftarrow  \tau y+(1-\tau)y',
\end{aligned}
\end{equation}
where $\tau$ is the synchronization rate, and $y$ represents the parameters $\theta$, $\phi$, and $\omega$. Finally, according to \cite{duan2021distributional,Haarnoja2018ASAC}, the entropy coefficient $\alpha$ is updated by minimizing the following objective
\begin{equation}
\label{eq.entropy_objective}
J(\alpha)=\mathbb{E}_{a\sim \pi_\omega}[-\alpha \log\pi(a|s;\omega)-\alpha\overline{\mathcal{H}}],
\end{equation}
where $\overline{\mathcal{H}}$ is the expected entropy. 

Finally, we proposed the E-DSAC algorithm according to the above analysis. The detail of our algorithm can be shown as Algorithm \ref{alg:PI-DSAC}.
\begin{algorithm}[!htb]
\caption{E-DSAC Algorithm}
\label{alg:PI-DSAC}
\begin{algorithmic}
\STATE Initialize parameters $\theta $, $\omega $, $\phi $ and entropy coefficient $\alpha$
\STATE Initialize target parameters $\theta '\leftarrow\theta$, $\omega '\leftarrow\omega $ and $\phi '\leftarrow\phi $
\STATE Initialize learning rate $\beta_{\SPACE{Z}}$, $\beta_{\pi}$, $\beta_{h}$, $\beta_{\alpha}$ and $\tau$
\STATE Initialize iterative step $k=0$
\REPEAT
\STATE Receive observation $\SPACE{O}$ and calculate state $\VECTOR{s}$ using \eqref{eq.pi_state}
\STATE Select action $\VECTOR{a}\sim\pi_{\omega }(\cdot|\VECTOR{s})$, observe $\SPACE{O}'$ and $r$ 
\STATE Store transition tuple $(\SPACE{O},\VECTOR{a},r,\SPACE{O}')$ in $\SPACE{B}$
\STATE
\STATE Randomly choose $N$ samples $(\SPACE{O},\VECTOR{a},r,\SPACE{O}')$ from  $\SPACE{B}$ 
\STATE Calculate states $s$ and $s'$ using \eqref{eq.pi_state} and obtain  the augmented tuple  $(\SPACE{O},\VECTOR{s},\VECTOR{a},r,\SPACE{O}',\VECTOR{s}')$ 
\STATE Update value network with
\eqref{eq05:value_gradient}: \\
\qquad \qquad $\theta  \leftarrow \theta  - \beta_{\SPACE{Z}}\nabla_{\theta }J_{\SPACE{Z}}(\theta ,\phi )$
\STATE Update feature network with \eqref{eq05:encode_gradient}: \\
\qquad \qquad $\phi  \leftarrow \phi  - \beta_{h}\nabla_{\phi }J_{\SPACE{Z}}(\theta ,\phi )$
\IF{$k \% m = 0$}
\STATE Update policy network with \eqref{eq04:policy_gradient_based_on_Q_PI}: \\ \qquad \qquad $\omega  \leftarrow \omega  + \beta_{\pi}\nabla_{\omega } J_{\pi}(\omega )$
\STATE Update target networks with \eqref{eq.target_update}
\STATE Update entropy coefficient $\alpha$ with \eqref{eq.entropy_objective}: \\
\qquad \qquad $\VECTOR{\alpha} \leftarrow \VECTOR{\alpha} - \beta_{\alpha}\nabla_{\VECTOR{\alpha}} J_{\alpha}(\VECTOR{\alpha})$
\ENDIF
\STATE $k=k+1$
\UNTIL Convergence 
\end{algorithmic}
\end{algorithm}

\begin{figure}[!htb]
\captionsetup{singlelinecheck = false,labelsep=period, font=small}
\centering{\includegraphics[width=0.48\textwidth]{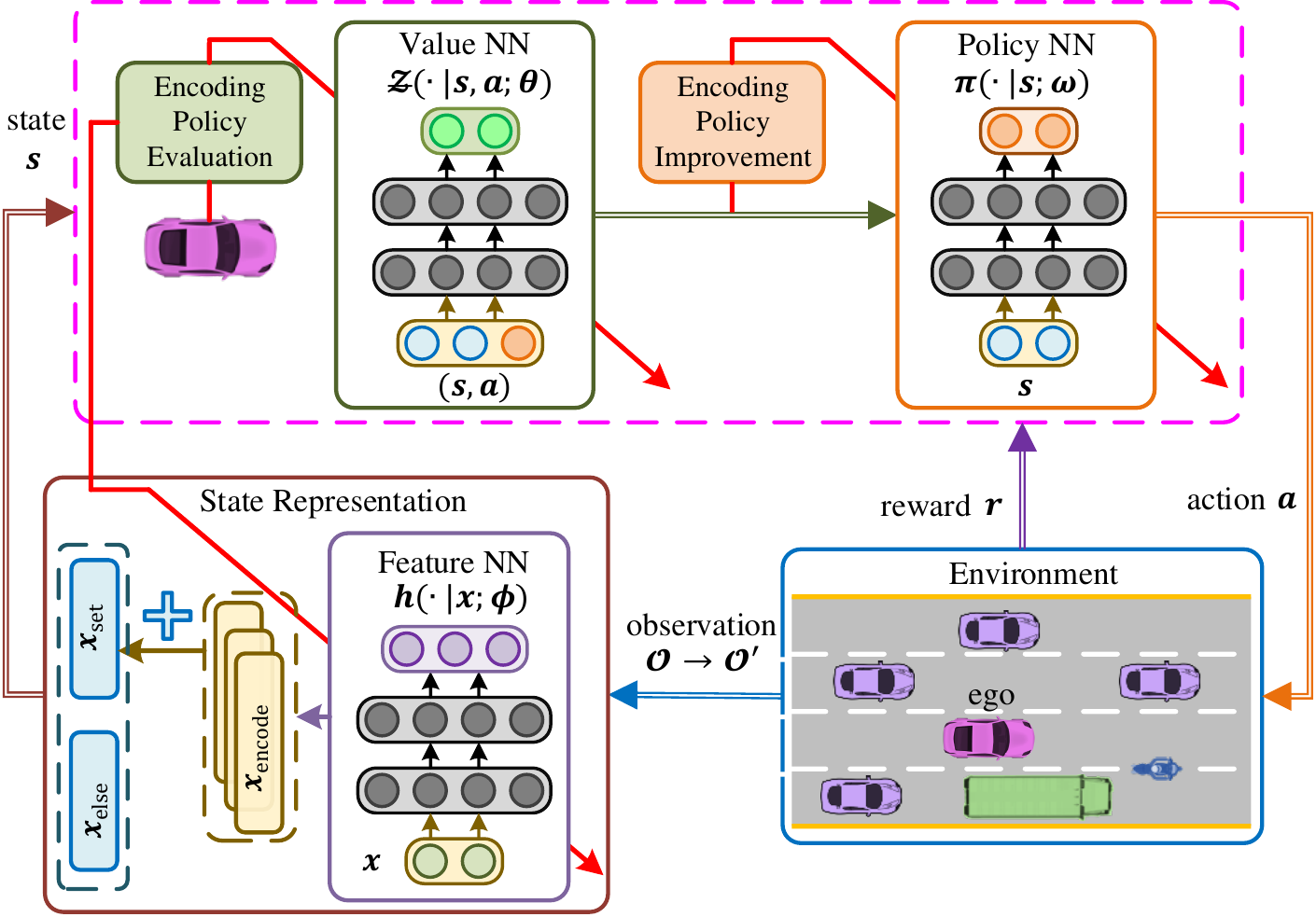}}
\caption{E-DSAC diagram. E-DSAC first updates the distributional value NN and feature NN based on the samples collected from the buffer. Then, the output of the value network is used to guide the update of the policy network.}
\label{f:E-DSAC}
\end{figure}

\section{Simulation Verification}
\label{sec:simulation}

This section validates the effectiveness of the proposed E-DSAC algorithm in a simulating multi-lane highway driving task.   

\subsection{Task Description}
As shown in Fig. \ref{f:scenario}, we built a one-way four-lane virtual road in SUMO \cite{SUMO2018} based on the 32km-long section of the Chinese Lianhuo Highway from point A (113$^\circ$29$'$03$''$E, 34$^\circ$51$'$48$''$N) to point B
(113$^\circ$48$'$53$''$E, 34$^\circ$49$'$53$''$N). Each lane is 3.75m wide. The speed limits of the four lanes from bottom to top are restricted to 60-100km/h, 80-100km/h, 90-120km/h, and 100-120km/h, respectively. The ego car would be initialized at a random position. To imitate real traffic situations, we arrange various motor vehicles on this road, including trucks, motorcycles, and cars of different sizes. All surrounding vehicles are initialized randomly at the beginning of each episode, and their movements are controlled by the car-following and lane-changing models of SUMO. The ego vehicle aims to ride as fast as possible without losing the guarantee of driving safety, regulations, and smoothness. 
\begin{figure}[!htb]
\captionsetup{singlelinecheck = false,labelsep=period, font=small}
\centering{\includegraphics[width=0.48\textwidth]{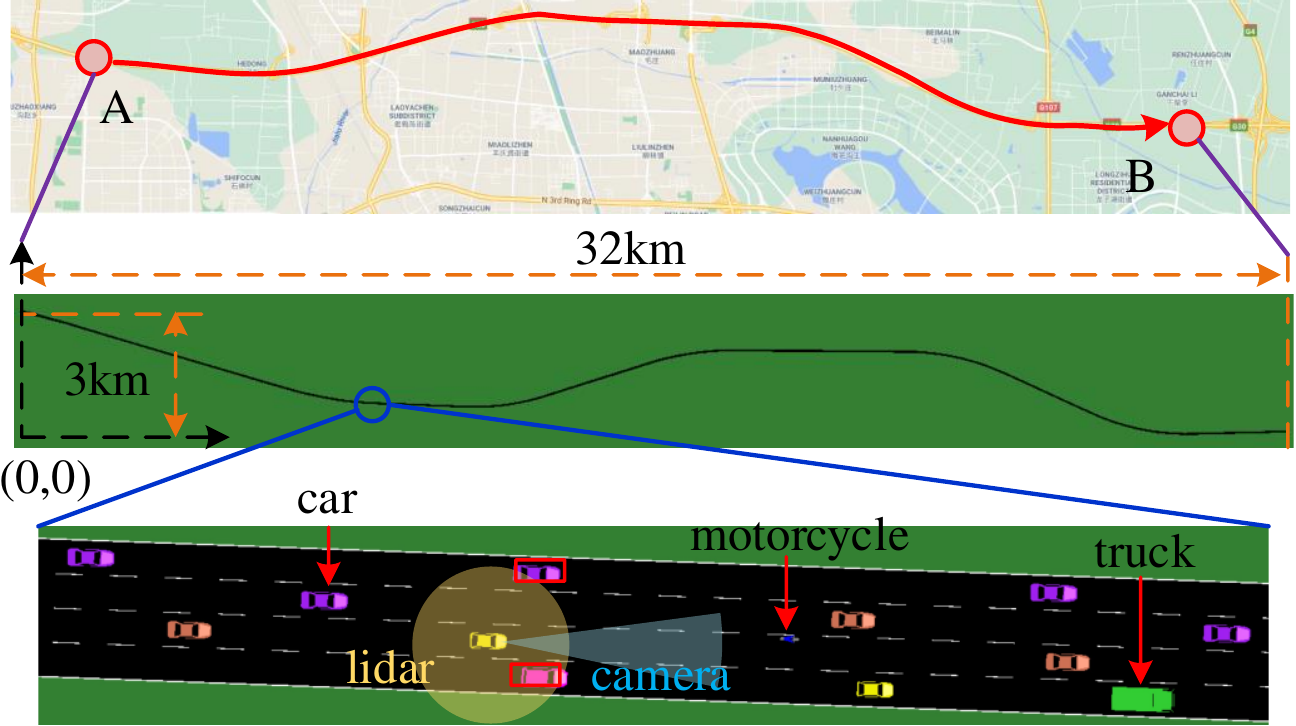}}
\caption{Illustration of the multi-lane highway driving scenario.}
\label{f:scenario}
\end{figure}

\subsection{Problem Formulation}
\subsubsection{Observation Design}
As mentioned before, the observation $\mathcal{O}$ consists of the information of surrounding vehicles and indicators related to the ego car and road geometry. As shown in Table \ref{tab:observation} and Fig. \ref{f:state_illustration}, we consider six indicators for each surrounding vehicle within the perception range,  including the longitudinal and lateral relative distance from the ego vehicle $D_{\rm long}$ and $D_{\rm lat}$, the relative speed $v_{\rm other}-v_{\rm ego}$, the heading angle relative to the lane centerline $\Phi_{\rm other}$, length $L_{\rm other}$, and width $W_{\rm other}$, i.e., $\VECTOR{x}=[D_{\rm long},D_{\rm lat},v_{\rm other}-v_{\rm ego},\Phi_{\rm other},L_{\rm other},W_{\rm other}]^\top$. In this simulation, the virtual sensor system of the ego car consists of a 360$^\circ$ lidar and a camera (see the shaded area in Fig. \ref{f:scenario}). According to the specifications of existing sensors such as HDL-32E and Mobileye 630 \cite{cao2020novel}, the effective range of the lidar is set to 80m, and that of the camera is set to 100m with a 38$^\circ$ horizontal field of view. Only the surrounding vehicles within the perception range and not blocked by other vehicles can be observed. Besides, each variable of surrounding vehicles is added with noise from a zero-mean Gaussian distribution before being observed. In particular, we assume the observation error of $x$ obeys $\mathcal{N}(0,{\rm diag}(0.14,0.14, 0.15,1,0.05,0.05))$. 

\begin{figure}[!htb]
\captionsetup{singlelinecheck = false,labelsep=period, font=small}
\centering{\includegraphics[width=0.48\textwidth]{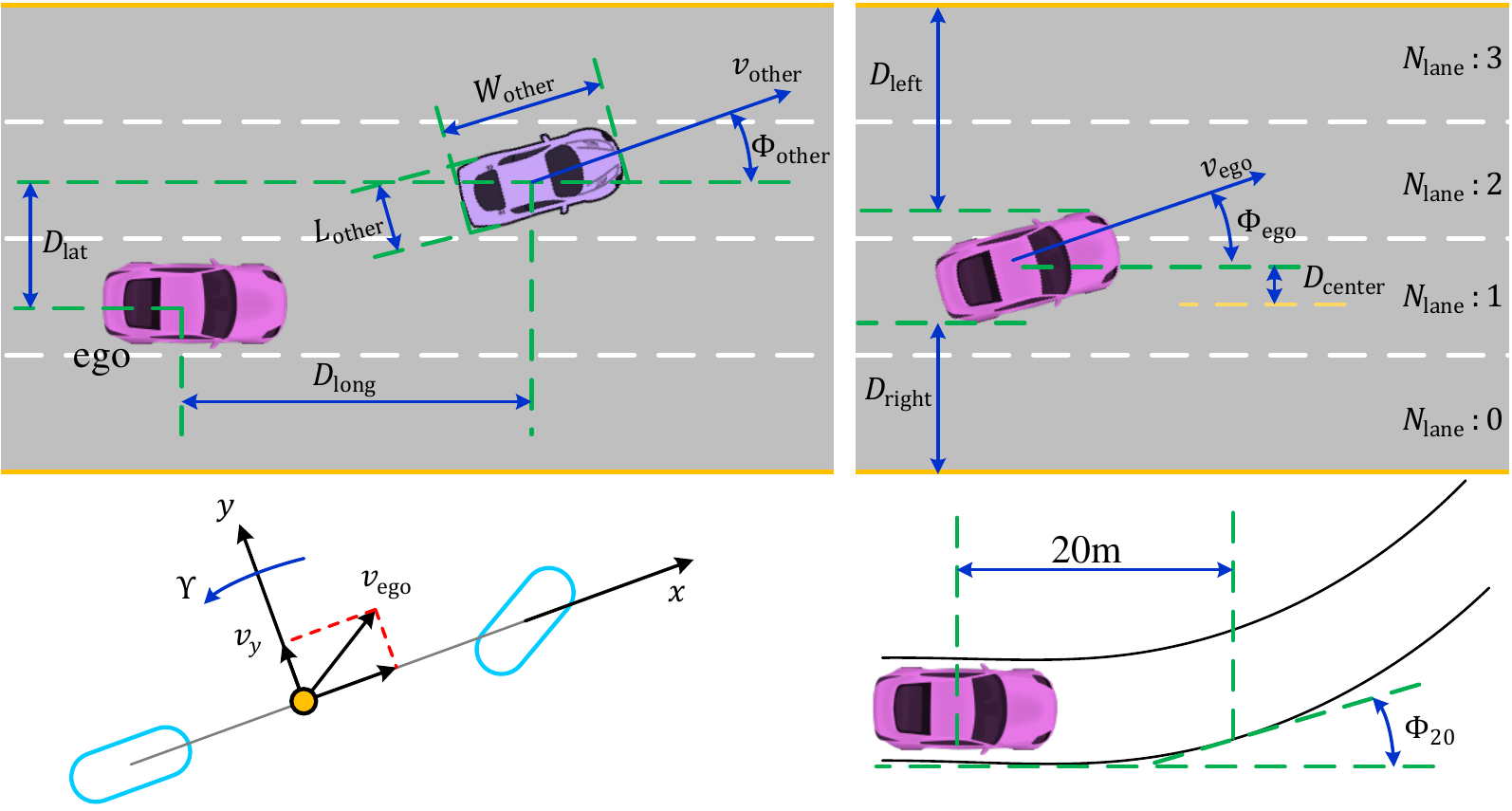}}
\caption{Illustration of some observation indicators.}
\label{f:state_illustration}
\end{figure}

In addition, $x_{\rm else}$ is designed as a 20-dimensional vector, of which 7 indicators are the information of the ego vehicle, namely vehicle speed $v_{\rm ego}$, lateral speed $v_y$, yaw rate $\Upsilon$, heading angle relative to the lane centerline $\Phi_{\rm ego}$, steering wheel angle $\Xi$, longitudinal acceleration $acc_{x}$, lateral acceleration $acc_{y}$. The other components of $x_{\rm else}$ are related to road geometry, including 
the distance from the ego to the lane centerline $D_{\rm center}$, distance to the left and right road edge $D_{\rm left}$, $D_{\rm right}$, the lane number $N_{\rm lane}$, the driving time in current lane $t_{\rm lanekeep}$, the difference between the ego speed and the upper and lower lane speed limits $v_{\rm upper}-v_{\rm ego}$, $v_{\rm ego}-v_{\rm lower}$. Furthermore, the lane direction changes at five different positions in front of the ego are used to describe the road shape, denoted as $\Phi_{\rm 10}$, $\Phi_{\rm 20}$, $\Phi_{\rm 30}$, $\Phi_{\rm 40}$ and $\Phi_{\rm 50}$. See Table \ref{tab:observation} and Fig. \ref{tab:observation} for more details.

\begin{table}[!htb]
\centering
\caption{Observation list}
\label{tab:observation}
\begin{tabular}{cccc}
\toprule
$\mathcal{O}$&Name & Symbol &Unit \\
\midrule
$x$&Longitudinal Distance&$D_{\rm long}$&m\\
&Lateral distance&$D_{\rm lat}$&m\\
&Relative speed&$v_{\rm other}-v_{\rm ego}$&m/s\\
&Relative heading angle&$\Phi_{\rm other}$& rad \\
&Length&$L_{\rm other}$& m \\
&Width&$W_{\rm other}$& m \\
\midrule
$x_{\rm else}$&Ego speed&$v_{\rm ego}$&m/s\\
&Lateral speed&$v_{y}$&m/s\\
&Yaw rate&$\Upsilon$&rad/s\\
&Heading angle of ego&$\Phi_{\rm ego}$&rad\\
&Steering wheel angle&$\xi$&rad\\
&Longitudinal acceleration&$acc_{x}$&m/${\rm{s}}^2$\\
&Lateral acceleration&$acc_{y}$&m/${\rm{s}}^2$\\
&Distance to centerline&$D_{\rm center}$&m\\
&Distance to left edge&$D_{\rm left}$&m\\
&Distance to right edge&$D_{\rm right}$&m\\
&Lane number&$N_{\rm lane}$& \\
&Speed difference to upper limit&$v_{\rm upper}-v_{\rm ego}$&m/s\\
&Speed difference to lower limit&$v_{\rm ego}-v_{\rm lower}$&m/s\\
&Lane-keeping time&$t_{\rm lanekeep}$&s\\
&Number of surrounding vehicles&$M$&\\
&Road direction change 10m ahead&$\Phi_{\rm 10}$&rad\\
&Road direction change 20m ahead&$\Phi_{\rm 20}$&rad\\
&Road direction change 30m ahead&$\Phi_{\rm 30}$&rad\\
&Road direction change 40m ahead&$\Phi_{\rm 40}$&rad\\
&Road direction change 50m ahead&$\Phi_{\rm 50}$&rad\\
\bottomrule
\end{tabular}
\end{table}

\subsubsection{Action Design} To prevent large discontinuity of the steering wheel angle, we choose the increment of the steering wheel angle and the expected longitudinal acceleration as actions, denoted as $\Delta\xi$ and $acc_{x,{\rm exp}}$, i.e, $\VECTOR{a} = [\Delta\xi,acc_{x,{\rm exp}}]^\top$. The expected steering wheel angle $\xi_{\rm exp}$ can be directly calculated as $\xi_{\rm exp}=\xi+\Delta\xi$. Since vehicles are controlled by saturated actuators, we let $\Delta\xi\in[-\frac{\pi}{9},\frac{\pi}{9}]$, $acc_{x,{\rm exp}}\in[-4,2]$ m/${\rm{s}}^2$. 

\subsubsection{Reward Design}
To realize reasonable autonomous driving on multi-lane highway, the reward function should consider driving efficiency, safety, regulations, and smoothness to guide the learning of driving policy. The reward function $r(\SPACE{O}, a, \SPACE{O}')$ can be expressed as
\begin{equation}
\label{eq05:reward_function}
r=\left\{\begin{array}{cc}-5000, &{\rm failure}\\R_{\rm speed}+R_{\rm smooth}+R_{\rm regulation}+R_{\rm safe},&{\rm else}\end{array}\right.,
\end{equation}
where $R_{\rm speed}$, $R_{\rm smooth}$, $R_{\rm regulation}$ and $R_{\rm safe}$ are the reward terms to encourage driving efficiency, smoothness, regulations, and safety, respectively. Moreover, -5000 is a large negative reward, which is used to punish some devastating events, including collision, driving off the road edge, and continuous lane changes in a short time (i.e., changing lanes when $t_{\rm lanekeep}< 3$).

$R_{\rm speed}$ is designed to encourage the ego vehicle to drive quickly by changing to high-speed lane:
\begin{equation}
\nonumber
\label{eq:speed_reward}
R_{\rm speed}=-0.6(v_{\rm max}-v_{\rm ego})^2,
\end{equation}
where $v_{\rm max}$ represents the maximum speed limit in all lanes, i.e., $v_{\rm max}=120$km/h. 

% In the sequel, the superscript $'$, such as $v_{\rm ego}'$, represents that this indicator belongs to the observation $\mathcal{O'}$ of next time step. 

$R_{\rm smooth}$ aims to make the driving process smoother and more comfortable by regularizing control variables $acc_x$, $\Delta \xi$ and certain states of the ego vehicle, expressed as
\begin{equation}
\label{eq:comfort_reward}
\nonumber
\begin{aligned}
R_{\rm smooth}=&-{acc_x}^2-5(acc_{x,{\rm exp}}-acc_x)^2-80{\xi_{\text{exp}}}^2
-300{\Delta \xi}^2 \\
&-500{\Phi_{\rm ego}}^2-30{v_y}^2-500{\Upsilon}^2-{acc_y}^2.
\end{aligned}
\end{equation}

The regulation term $R_{\rm rule}$ encourages the ego vehicle to comply with driving rules
\begin{equation}
\label{eq:legal_reward}
\nonumber
\begin{aligned}
R_{\rm rule}= &\underbrace{-10{D_{\rm center}}^2-40(1-\tanh(4\min\{D_{\rm left},D_{\rm right}\}))}_{\text {deviation punishment}}\\
&\underbrace{-{\rm sgn}(v_{\rm ego}-v_{\rm upper})(v_{\rm ego}-v_{\rm upper})^2}_{\text {overspeed punishment}}\\
&\underbrace{-{\rm sgn}(v_{\rm lower}-v_{\rm ego})(v_{\rm lower}-v_{\rm ego})^2}_{\text {underspeed punishment}}.
\end{aligned}
\end{equation}
where $\rm sgn(\cdot)$ denotes the sign function, i.e., ${\rm sgn}(a)=1$ if $a\ge 0$, ${\rm sgn}(a)=0$, otherwise.

Last but not least, $R_{\rm safe}$ aims to improve driving safety by penalizing the distance between vehicles that are too close. As shown in Fig. \ref{f:risk_area}, we first introduce the lateral gap $D_{\rm LatGap}$ and longitudinal gap $D_{\rm LongGap}$ between the ego vehicle and each surrounding vehicle, expressed as 
\begin{equation}
D_{\rm LatGap}=|D_{\rm lat}|-\frac{W_{\rm other}+W_{\rm ego}}{2},
\end{equation}
and
\begin{equation}
D_{\rm LongGap}=|D_{\rm long}|-\frac{L_{\rm other}+L_{\rm ego}}{2},
\end{equation}
where $L_{\rm ego}$ and $W_{\rm ego}$ are the length and width of the ego vehicle. When $D_{\rm LatGap}\le 0$, a rear-end collision may occur between the vehicle and the corresponding surrounding vehicle. This means that we need to impose penalties if the longitudinal gap is too small in this case. On the other hand, when $D_{\rm LongGap}\le 0$, a side collision may occur, and penalties are required if the lateral gap is too small. Therefore, $R_{\rm safe}$ is designed as 
\begin{equation}
\label{eq:safe_reward}
\nonumber
\begin{aligned}
&R_{\rm safe}=\underbrace{70}_{\text{single-step incentive}}  \\
&\underbrace{-40\sum_{\VECTOR{x}\in\SPACE{X}}{\rm sgn}(-D_{\rm LatGap}){\rm sgn}(D_{\rm long})\Big(1-\tanh{\frac{D_{\rm LongGap}}{v_{\rm ego}}}\Big)}_{\text{Penalty for being close to cars ahead} }\\
&\underbrace{-25\sum_{\VECTOR{x}\in\SPACE{X}}{\rm sgn}(-D_{\rm LatGap}){\rm sgn}(-D_{\rm long})\Big(1-\tanh{\frac{D_{\rm LongGap}}{\VECTOR{v_{\rm other}}}}\Big)}_{\text{Penalty for being close to cars behind}}\\
&\underbrace{-40\sum_{\VECTOR{x}\in\SPACE{X}}{\rm sgn}(-D_{\rm LongGap})\Big(1-\tanh(1.5D_{\rm LatGap})\Big)}_{\text{Penalty for being close to side cars}}.
\end{aligned}
\end{equation}
\begin{figure}[!htb]
\captionsetup{singlelinecheck = false,labelsep=period, font=small}
\centering{\includegraphics[width=0.4\textwidth]{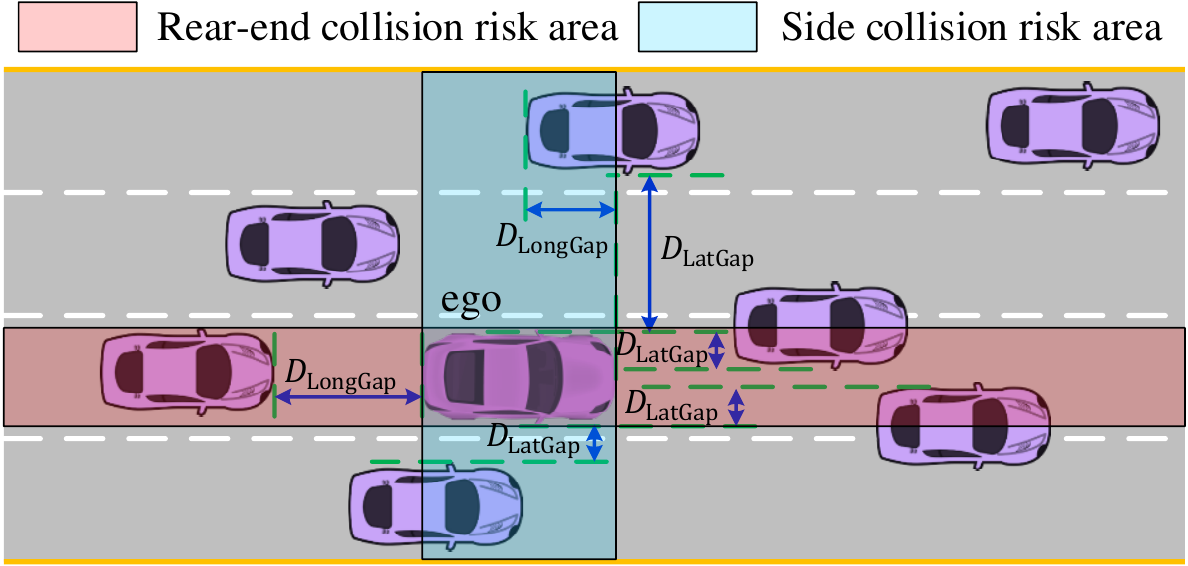}}
\caption{Collision risk area.}
\label{f:risk_area}
\end{figure}

\subsection{Algorithm Details}
Based on the above problem setting, we validate the effectiveness of the proposed E-DSAC by comparing it with the DSAC algorithm that takes $U_{\rm FP}(\SPACE{O})$ in \eqref{eq.mapping_s} as states. Since $U_{\rm FP}(\SPACE{O})$ is permutation sensitive and only suitable for a fixed number of surrounding vehicles, DSAC only considers the nearest 6 vehicles, which are arranged in increasing order according to relative distance. The proposed E-DSAC algorithm is implemented in two settings: (1) considering the nearest 6 vehicles, i.e., $M=6$; (2) considering all observed surrounding vehicles within the perception region, i.e., $M\in[1,N]$ is constantly changing.

The value function, policy, and feature NN adopt almost the same NN architecture, which contains 5 hidden layers, with 128 units per layer. All hidden layers take Gaussian Error Linear Units (GELU)  \cite{hendrycks2016gelu} as activation functions. According to Lemma \ref{lemma.encoding}, the output dimension $d_3$ of $h_{\phi}$ should satisfy that $d_3\ge Nd_1+1$. We assume the maximum number of surrounding vehicles with perception region is 20, i.e., $N=20$, so we set $d_3=121$. The Adam method \cite{Diederik2015Adam} with a cosine annealing learning rate is used to update all NNs. See Table \ref{tab:hyper} for more training details.
\begin{table}[!htb]
\centering
\caption{Training hyper-parameters}
\label{tab:hyper}
\begin{tabular}{ccc}
\toprule
 Name & symbol &Value \\
\hline
 Batch size& & 256\\
 Value learning rate& $\beta_{\SPACE{Z}}$ & $8\times 10^{-5}\rightarrow 4\times 10^{-5}$\\
 Feature learning rate & $\beta_h$&$8\times 10^{-5}\rightarrow 4\times 10^{-5}$\\
Policy learning rate & $\beta_\pi$&$5\times 10^{-5}\rightarrow 4\times 10^{-5}$\\
 Learning rate of $\alpha$ &$\beta_\alpha$& $1\times 10^{-4}\rightarrow4\times 10^{-5} $ \\
Discounted factor&$\gamma$ & 0.99\\
Update delay & $m$&2\\
 Target update rate & $\tau$&0.001\\
 Target entropy  &  $\overline{\SPACE{H}}$&-2\\
\bottomrule
\end{tabular}
\end{table}

\subsection{Results}
For each case, we train five different runs with evaluations every 20000 iterations. We evaluate the driving performance by calculating the average return over five episodes at each evaluation, where the maximum length of each episode is 500 time steps. Fig.~\ref{f:return} demonstrates the learning curves of E-DSAC and DSAC. In the case of considering the 6 nearest surrounding vehicles, the final policy performance learned by E-DSAC (14398.46$\pm$177.67) is about 
three times that of DSAC (4556.10$\pm$1134.72). The only difference between DSAC and E-DSAC when $M=6$ is that E-DSAC makes decisions based on permutation invariant state representation with the help of feature NN. In addition to the performance gains, E-DSAC also eliminates the requirement of manually designing sorting rules $o$ in $U(\mathcal{O})$. On the other hand, the E-DSAC ($M\in[1,N]$) that considers all observed surrounding vehicles slightly outperforms E-DSAC ($M=6$). This indicates that the proposed E-DSAC algorithm is capable of handling a variable number of surrounding vehicles. Although 6 surrounding vehicles seem to meet the basic needs of self-driving in this simulation, such a number may not be enough for other scenarios such as complex intersections. The capability of E-DSAC to process variable-size sets makes it more suitable for different driving tasks. To conclude, the state encoding module adopted by E-DSAC significantly improves the policy performance, relaxing the requirement for predetermined sorting rules and vehicle number restrictions.

\begin{figure}[!htb]
\captionsetup{singlelinecheck = false,labelsep=period, font=small}
\centering{\includegraphics[width=0.4\textwidth]{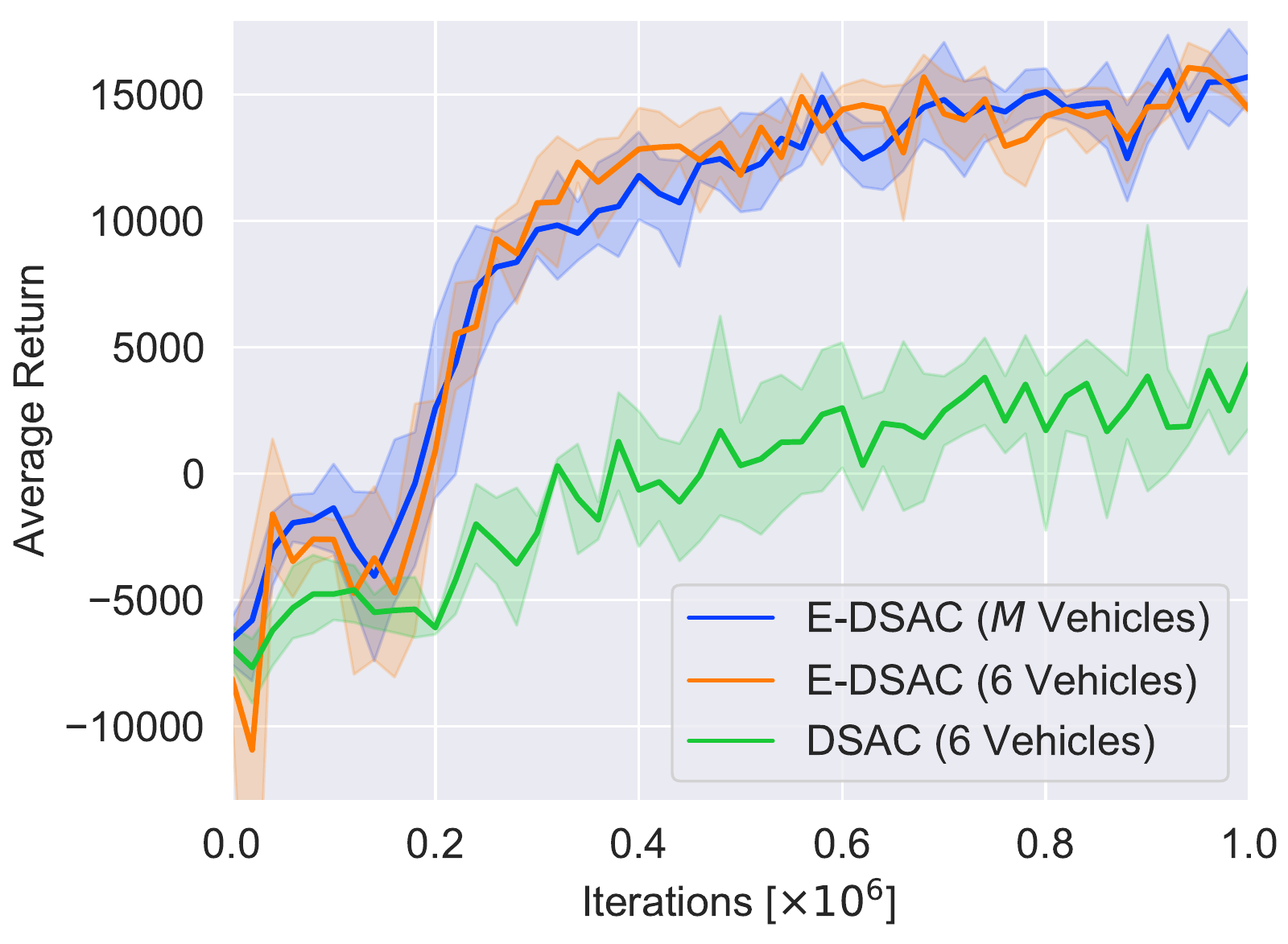}}
\caption{Return comparison. The solid lines correspond to the mean and the shaded regions correspond to 95\% confidence interval over 5 runs.}
\label{f:return}
\end{figure}

Next, we will analyze the driving behavior of the learned policy. As the final policy performance of E-DSAC ($M\in[1,N]$) and E-DSAC ($M=6$) is similar, only the former will be considered in the sequel. In addition to E-DSAC and DSAC, we have also introduced a rule-based baseline, in which we control the ego vehicle using the Krauss car-following and SL2015 lane-changing models of SUMO \cite{SUMO2018}. Firstly, for each method, we performed 100 simulations starting from a random location of the rightmost lane, i.e., low-speed lane. The maximum time length of each simulation is 100 seconds. Assuming the average speed of ego is 100km/h, the total simulated driving distance is about 300km. Fig. \ref{f:avg_speed} shows the average speed of the ego vehicle, the preceding vehicle, and all observed surrounding vehicles (i.e., traffic) during the simulation. Results show that the policy learned by E-DSAC maintains a high level of driving efficiency. Its average speed (113.04$\pm$6.55 km/h) is about 13.24km/h higher than DSAC, and 11.10km/h higher than SUMO. It is worth noting that only the ego car controlled by E-DSAC rides at a speed faster than the average speed of the traffic flow (about 11.46km/h higher), whose speed is also 2.85km/h higher than the preceding vehicle. This means that the self-driving car has learned to speed up by changing to high-speed lanes or overtaking.  As a comparison, the speed of the ego car controlled by DSAC is lower than traffic and its preceding vehicle. This indicates the DSAC policy tends to make conservative decisions, such as following the vehicle ahead at a lower speed.
\begin{figure}[!htb]
\captionsetup{singlelinecheck = false,labelsep=period, font=small}
\centering{\includegraphics[width=0.4\textwidth]{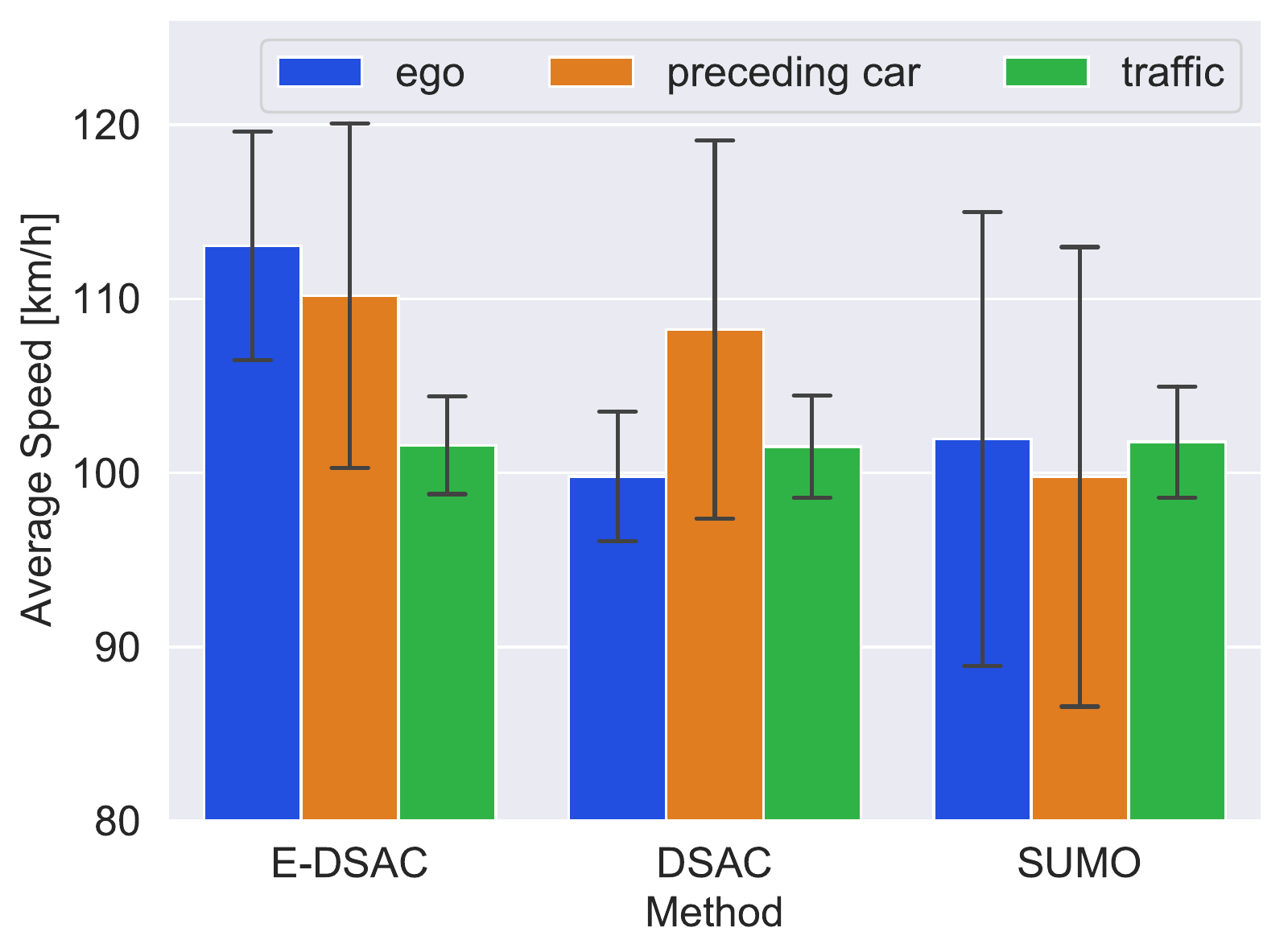}}
\caption{Average speed comparison. The error bar represents the standard deviation.}
\label{f:avg_speed}
\end{figure}

Fig. \ref{f:lanechange_compare} compares the changes in the following distance and the speed of the preceding vehicle before and after lane change. It can be seen that after changing lanes, the following distance and the speed of the preceding vehicle have increased by 26.11m and 13.14km/h, respectively. This indicates that the policy has learned to actively change into lanes with sparse traffic and faster speeds, which further explains the high driving efficiency of E-DSAC shown in Fig. \ref{f:avg_speed}.
\begin{figure}[!htb]
\centering
\captionsetup[subfigure]{justification=centering}
\subfloat[Distance comparison]{\includegraphics[width=0.49\linewidth]{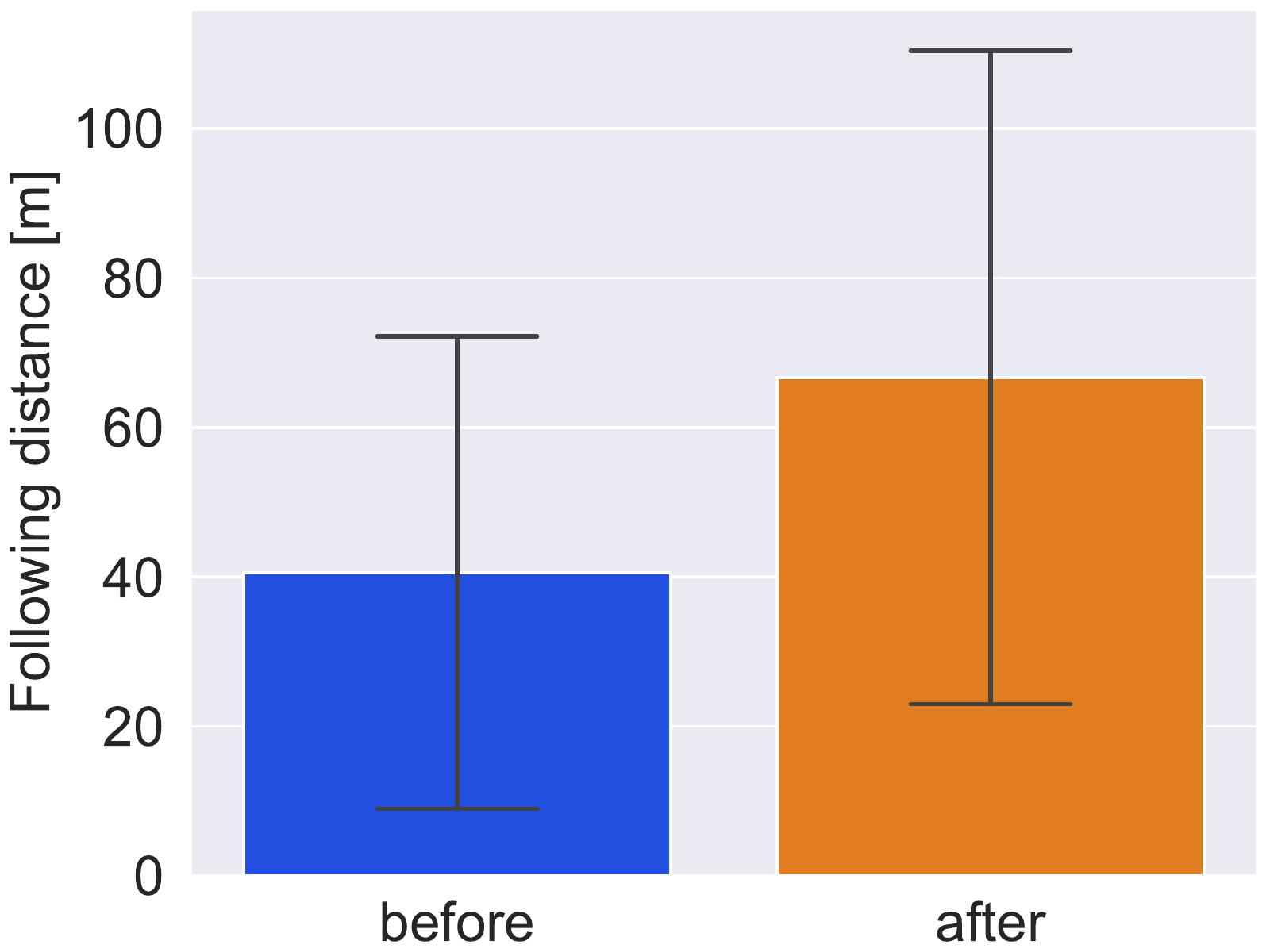}}
\subfloat[Speed comparison]{\includegraphics[width=0.49\linewidth]{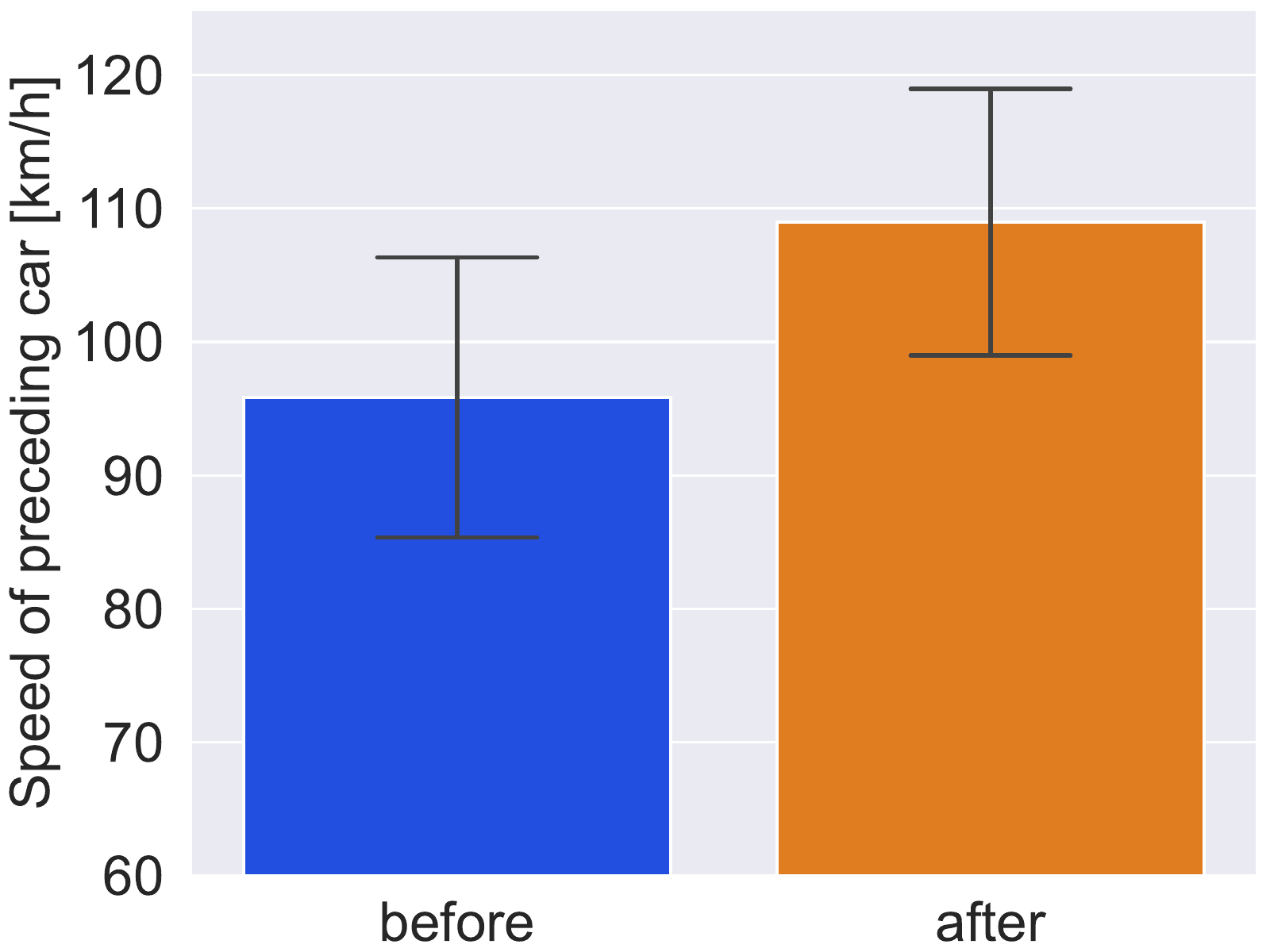}}
\\
\caption{Difference before and after lane changing. The error bar represents the standard deviation.}
\label{f:lanechange_compare}
\end{figure} 

A typical driving process and corresponding state curves are visualized in Fig.~\ref{f:e1_simu_traj} and \ref{f:e1_state}, respectively. As shown in Fig. \ref{fig:e1_tra_1}, \ref{fig:e1_tra_2} and \ref{fig:e1_velocity}, the ego vehicle is initialized in the outermost lane, and then the first left lane change is completed after going straight for about 100m. During this process, the vehicle speed is accelerated from 90km/h to about 100km/h. Next, the ego vehicle goes straight for a period of time, and then completes the second left lane change, increasing the driving speed to about 110km/h (See Fig. \ref{fig:e1_tra_3} and \ref{fig:e1_tra_4}). Throughout the entire driving process, the ego vehicle merges from the low-speed lane into the high-speed lane through two left lane changes, while maintaining a reasonable distance from the surrounding traffic. It can also be seen from Fig. \ref{fig:e1_control} and \ref{fig:e1_angle} that the control inputs and state curves are very smooth. Appendix \ref{appen.example} gives an instance of the overtaking process. These two cases show that the learned policy is capable of finding a reasonable lane-change position and timing through acceleration and deceleration.

\begin{figure}[!htb]
\captionsetup{singlelinecheck = false,labelsep=period, font=small}
\centering
\captionsetup[subfigure]{justification=centering}
\subfloat[Initialization]{\label{fig:e1_tra_1}\includegraphics[width=0.99\linewidth]{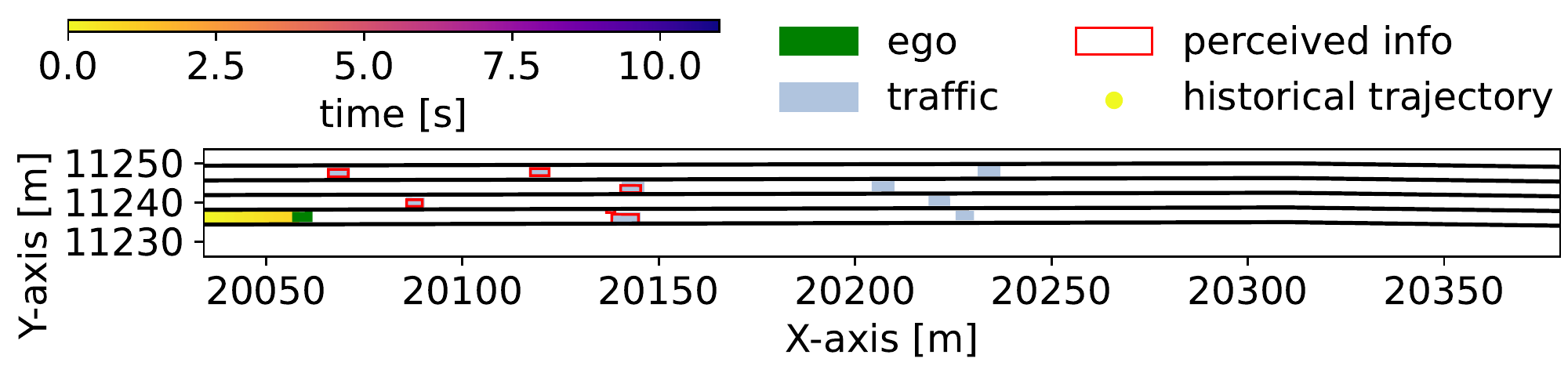}}
\\
\subfloat[The 1st lane change]{\label{fig:e1_tra_2}\includegraphics[width=0.99\linewidth]{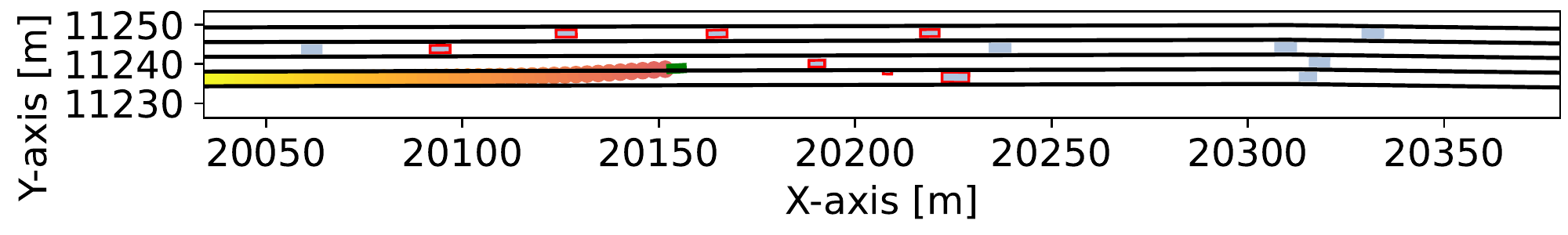}}
\\
\subfloat[The 2nd lane change]{\label{fig:e1_tra_3}\includegraphics[width=0.99\linewidth]{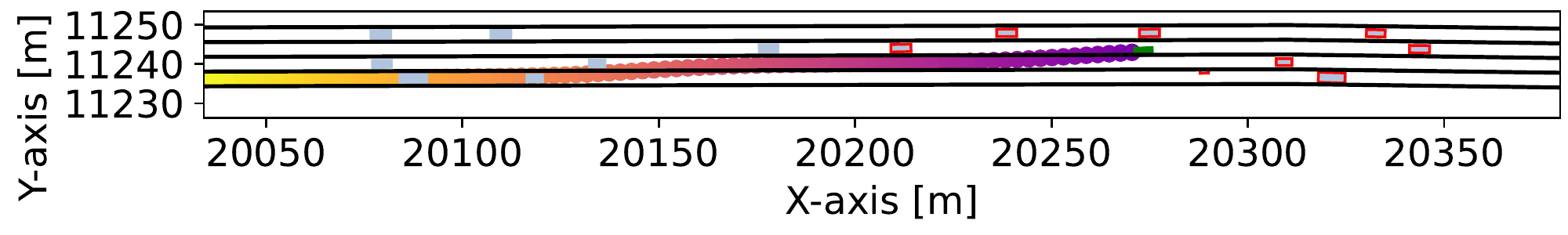}}
\\
\subfloat[Going straight]{\label{fig:e1_tra_4}\includegraphics[width=0.99\linewidth]{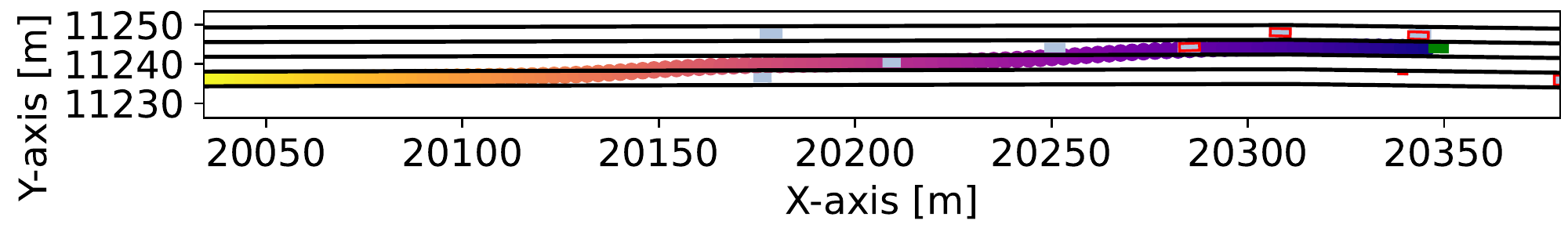}}\\
\caption{Simulation 1: trajectory. The red box represents the indicators of the perceived vehicle.}
\label{f:e1_simu_traj}
\end{figure} 

\begin{figure}[!htb]
\centering
\captionsetup[subfigure]{justification=centering}
\subfloat[Action commands]{\label{fig:e1_control}\includegraphics[width=0.99\linewidth]{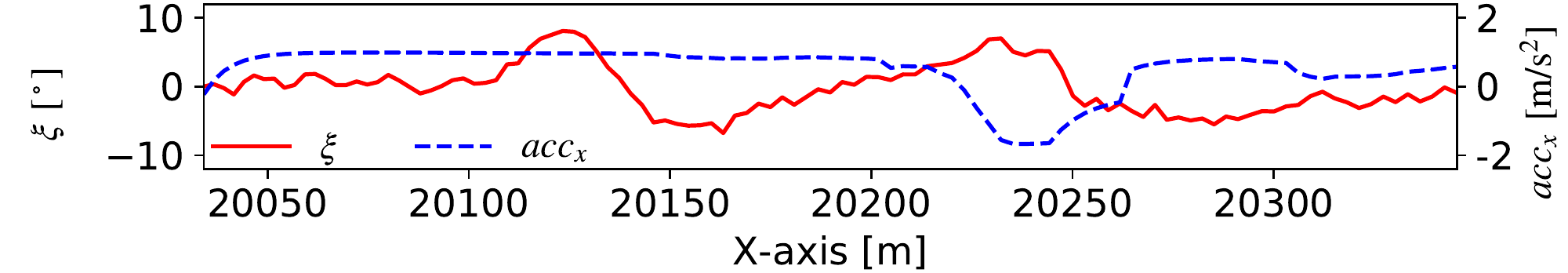}}\\
\subfloat[Heading angle and yaw rate]{\label{fig:e1_angle}\includegraphics[width=0.99\linewidth]{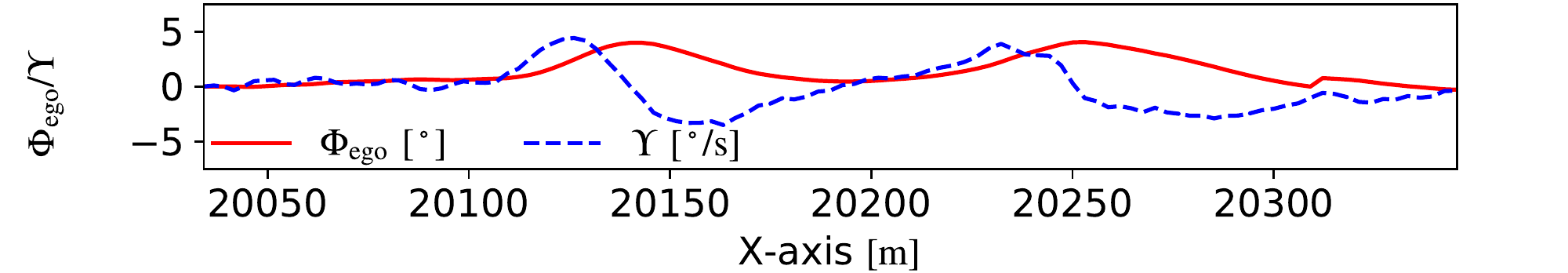}}\\
\subfloat[Speed]{\label{fig:e1_velocity}\includegraphics[width=0.99\linewidth]{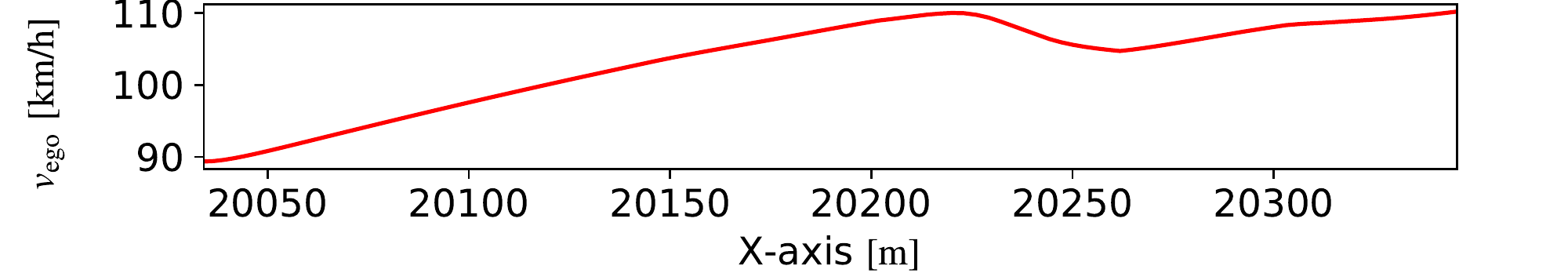}}\\
\caption{Simulation 1: state and action curves}
\label{f:e1_state}
\end{figure} 

In Fig. \ref{f:distance_distribution}, we display the joint distribution of the following distance and the distance between ego and its following vehicle. Results show that the ego car can maintain a reasonable distance from surrounding vehicles during the entire simulation process, providing evidence for the safety of the E-DSAC policy. In summary, E-DSAC outperforms DSAC by a wide margin in terms of policy performance in the field of autonomous driving. The policy learned by E-DSAC can realize efficient, smooth, and relatively safe autonomous driving in the designed multi-lane highway scenario. Besides, Appendix \ref{appen.extension} demonstrates that the proposed encoding policy iteration framework can also be extended to other RL algorithms such as SAC \cite{Haarnoja2018ASAC}, TD3 \cite{Fujimoto2018TD3}, and DDPG \cite{lillicrap2015DDPG}. 
\begin{figure}[!htb]
\captionsetup{singlelinecheck = false,labelsep=period, font=small}
\centering{\includegraphics[width=0.4\textwidth]{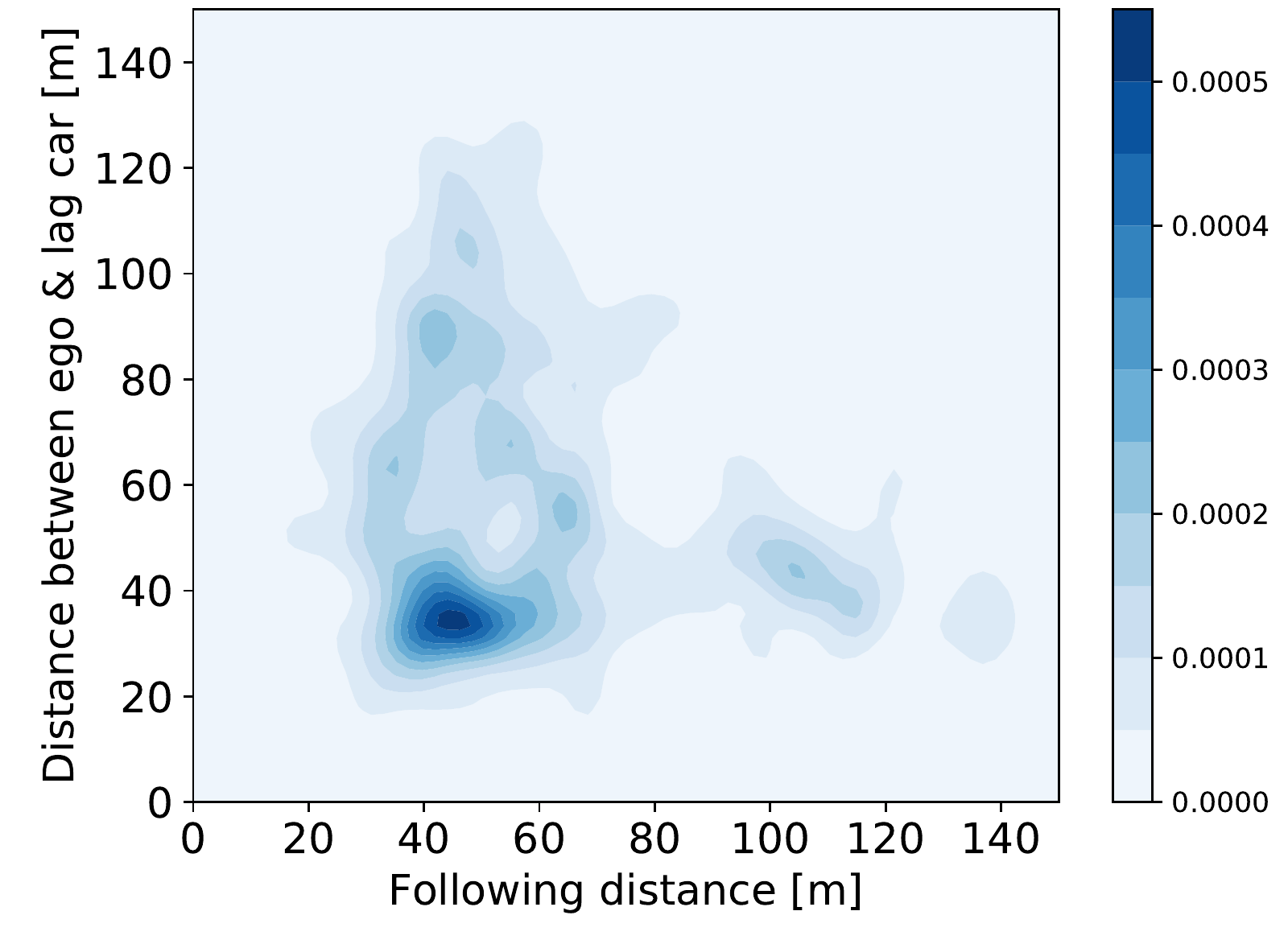}}
\caption{Distance distribution.}
\label{f:distance_distribution}
\end{figure}

\section{Real Vehicle Test}
\label{sec:real_veh_test}
In this section, we deploy the learned policy on a real automated vehicle on a two-lane park road to verify the effectiveness of E-DSAC in practical applications. 
\subsection{Experiment Description}
The experiment road located at ($36^{\circ}18'20''$N, $120^{\circ}40'25''$E) in Suzhou, China, has a total length of $170$m, and each lane is 3.5m wide. There are two speed bumps on this road. The test vehicle is a Chang-an CS55 SUV, equipped with speed following and steering wheel angle tracking systems. An industrial PC (IPC) is employed to send control commands. The indicators of ego vehicle can be obtained through RTK modules and CAN bus, and the 51Sim-One traffic simulation software is adopted to generate continuous surrounding traffic. 

The experiment pipeline is shown in Fig.~\ref{f:architeture}. At first, we used E-DSAC in SUMO to train a decision-making module composed of policy and feature NNs offline, and then deployed it in IPC. At each time step, the decision module will output the corresponding control commands, according to the observation information received by IPC. Noted that in this experiment, we only consider the lateral decision-making and control, so only the expected steer angle command is sent to the steering wheel angle tracking system through CAN bus. For longitudinal control, the speed tracking system is used to track the expected speed, which is 20km/h. The autonomous driving process will continue until the end of the road.
\begin{figure}[!htb]
\captionsetup{singlelinecheck = false,labelsep=period, font=small}
\centering{\includegraphics[width=0.99\linewidth]{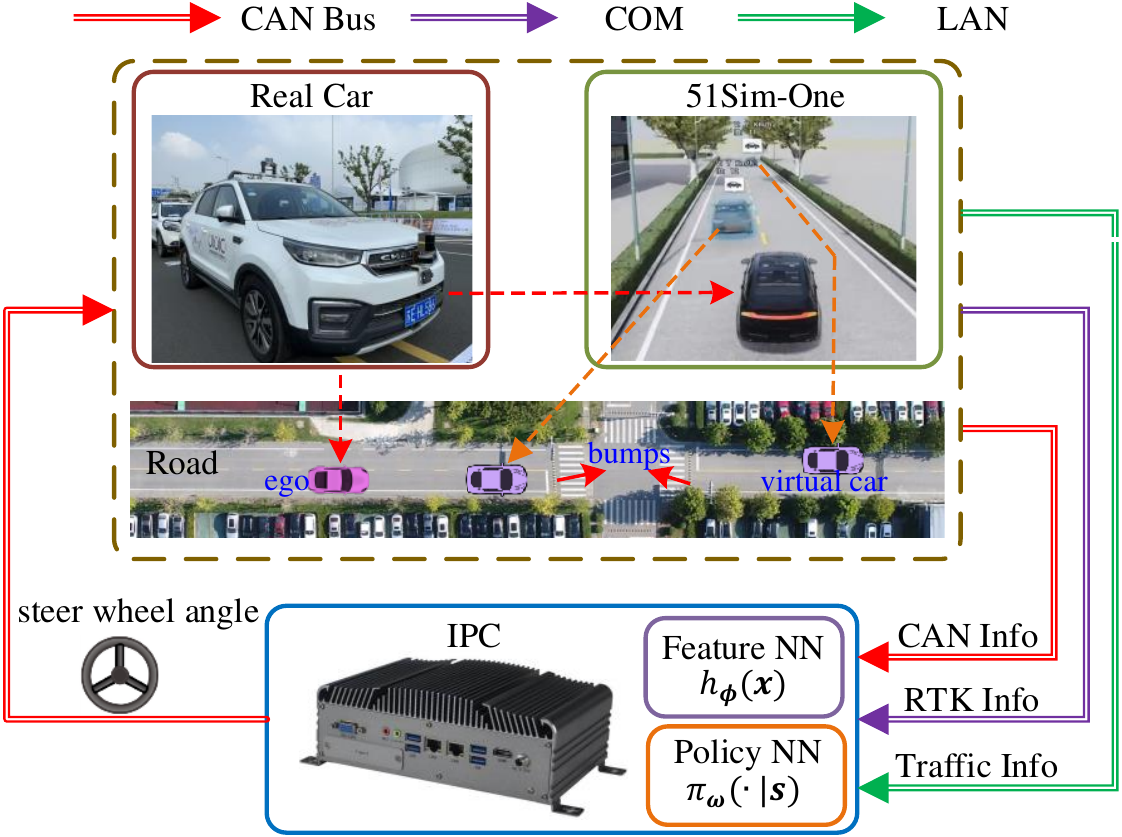}}
\caption{Real vehicle test architecture.}
\label{f:architeture}
\end{figure}

\subsection{Experimental Results}

Fig. ~\ref{f: experiment_example_1} shows a typical driving process and the corresponding state and action curves. There were two surrounding vehicles on the road with a speed of about 7.2km/h. From Fig. \ref{fig:experi_tra}, the ego vehicle started from the right lane, and immediately chose to change to the left lane due to the slow speed of the preceding car. After that, it gradually approached another front car in the left lane. Next, the ego vehicle changed to the right lane to avoid collision, and finally drove to the end. Fig. \ref{fig:experi_control} displays the expected and actual steering wheel angle curves during driving. Although there exists a system response delay of about 0.1s between the expected value and the actual, the actual steering wheel angle changed smoothly during the whole driving process, and only a small oscillation occurred when the vehicle passes through two bumps. Similarly, the yaw rate and heading angle maintained a smooth trend during riding in Fig.~\ref{fig:experi_state}, which indicates that the proposed algorithm can assure a satisfactory level of driving comfort. Besides, as shown in Fig.~\ref{fig:experi_deviation}, the vehicle mainly ran near the centerline while going straight.
\begin{figure}[thpb]
\centering
\captionsetup{singlelinecheck = false,labelsep=period, font=small}
\captionsetup[subfigure]{justification=centering}
\subfloat[Trajectory]{\label{fig:experi_tra}\includegraphics[width=0.9\linewidth]{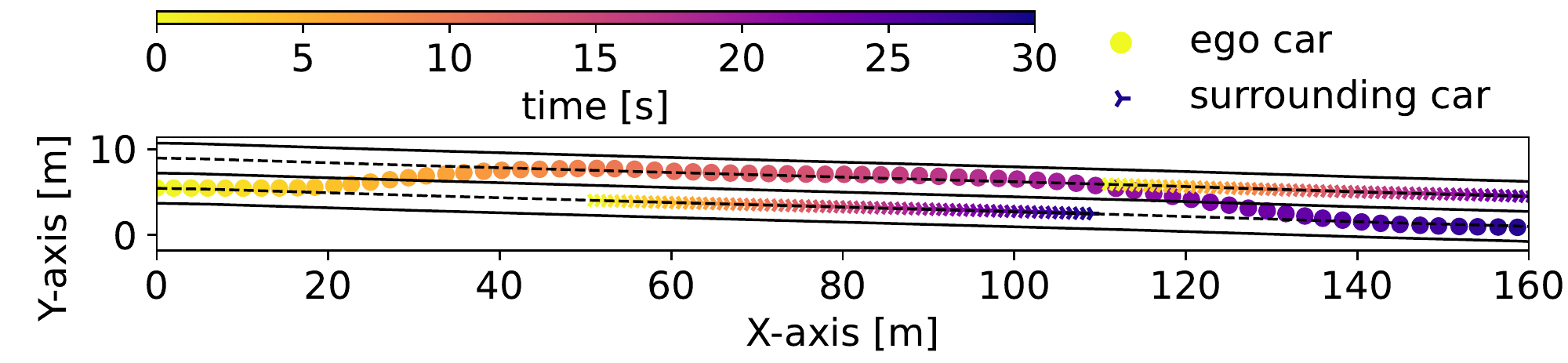}}
\\
\subfloat[Steering wheel angle]{\label{fig:experi_control}\includegraphics[width=0.9\linewidth]{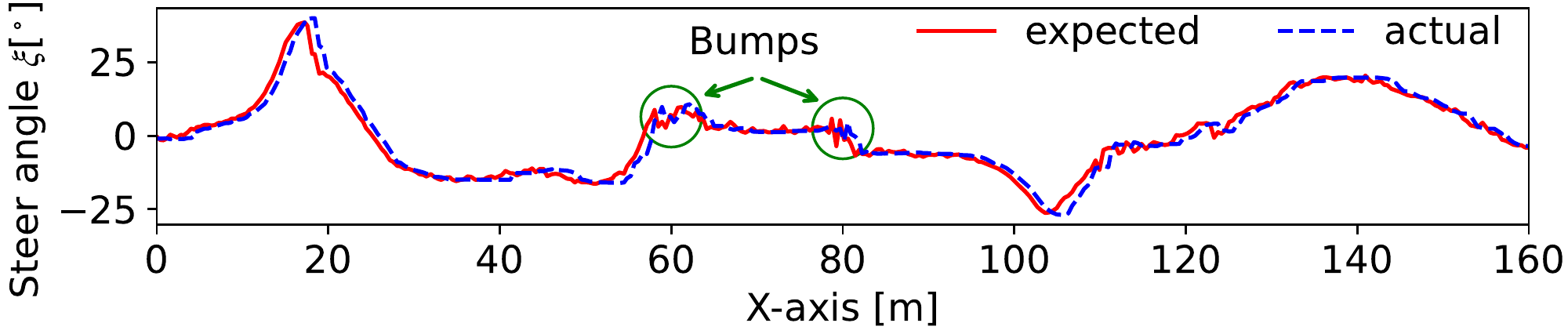}}
\\
\subfloat[Heading angle and yaw rate]{\label{fig:experi_state}\includegraphics[width=0.9\linewidth]{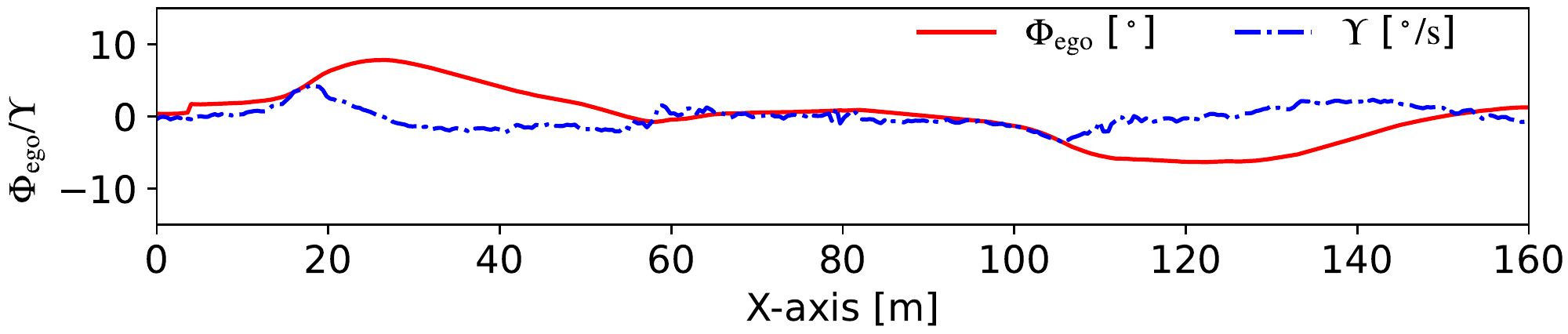}}
\\
\subfloat[Distance to lane center]{\label{fig:experi_deviation}\includegraphics[width=0.9\linewidth]{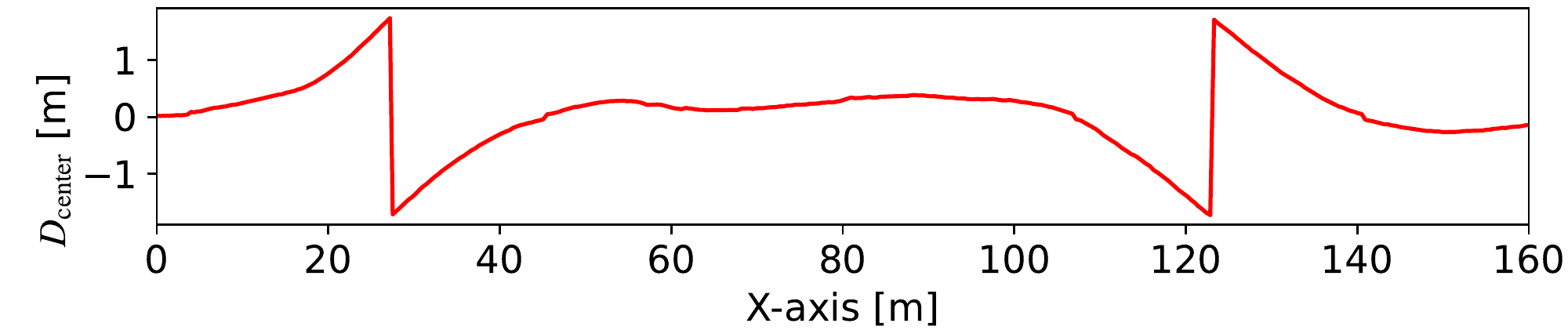}}
\caption{Experiment 1: trajectory, state and action curves.}
\label{f: experiment_example_1}
\end{figure}

Fig. ~\ref{f: experiment_tra} shows five trajectories of autonomous driving under different initial conditions. In Fig. \ref{fig:tra_1}, the speeds of vehicles ahead in the right and left lanes were about 7.2km/h and 20km/h, respectively. The ego car first changed to the left lane, and then went straight until the end. In Fig. \ref{fig:tra_2}, the ego car turned to the left lane to avoid collision when it found that the vehicle ahead (7.2km/h) in the left lane suddenly changed lane to the right. After that, a right lane change was made to realize continuous collision avoidance. In Fig. \ref{fig:tra_3}, the ego car headed straight to the end due to the fast speed (20km/h) of the preceding car. In Fig. \ref{fig:tra_4}, the ego car has changed lanes three times in a row to avoid the slow-moving vehicle ahead (3.6km/h). In Fig. \ref{fig:tra_5}, the ego bypassed two stationary cars and drove towards the destination. Combining Fig. \ref{f: experiment_example_1} and \ref{f: experiment_tra}, experimental results show that the learned policy of E-DSAC can smoothly complete maneuvers such as lane-keeping, lane changing and overtaking, so as to realize autonomous driving in response to different surrounding vehicles. 

\begin{figure}[thpb]
\captionsetup{singlelinecheck = false,labelsep=period, font=small}
\centering
\captionsetup[subfigure]{justification=centering}
\subfloat[]{\label{fig:tra_1}\includegraphics[width=0.9\linewidth]{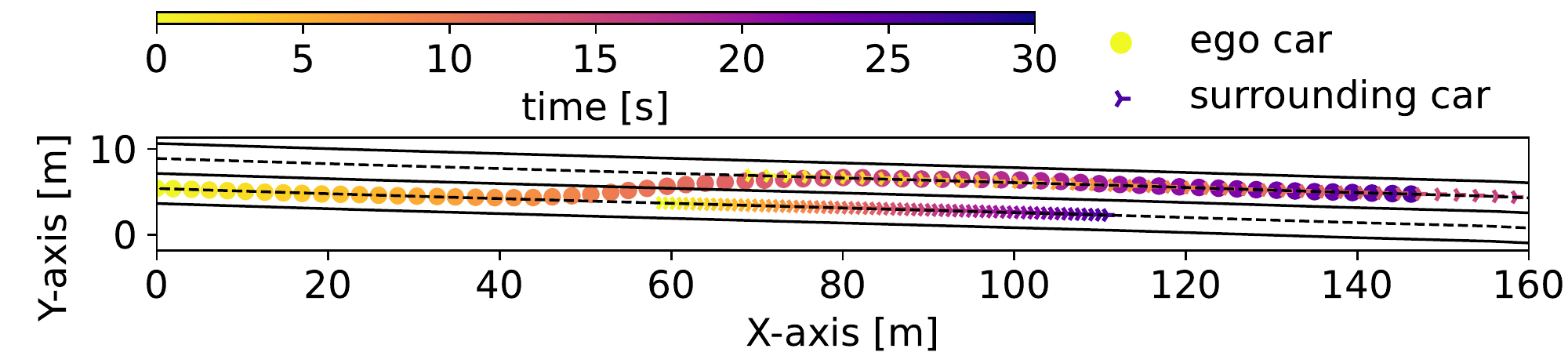}}
\\
\subfloat[]{\label{fig:tra_2}\includegraphics[width=0.9\linewidth]{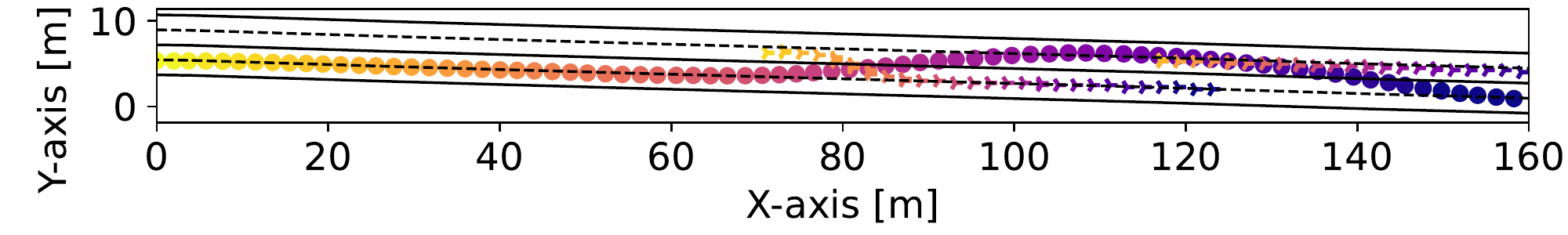}}
\\
\subfloat[]{\label{fig:tra_3}\includegraphics[width=0.9\linewidth]{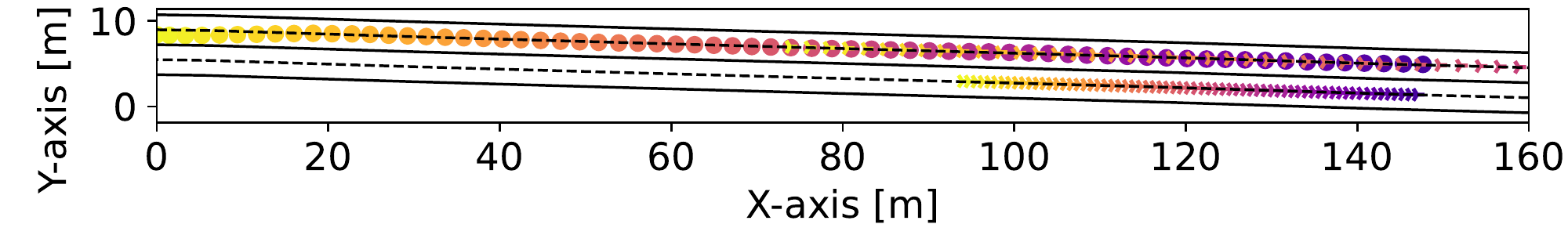}}
\\
\subfloat[]{\label{fig:tra_4}\includegraphics[width=0.9\linewidth]{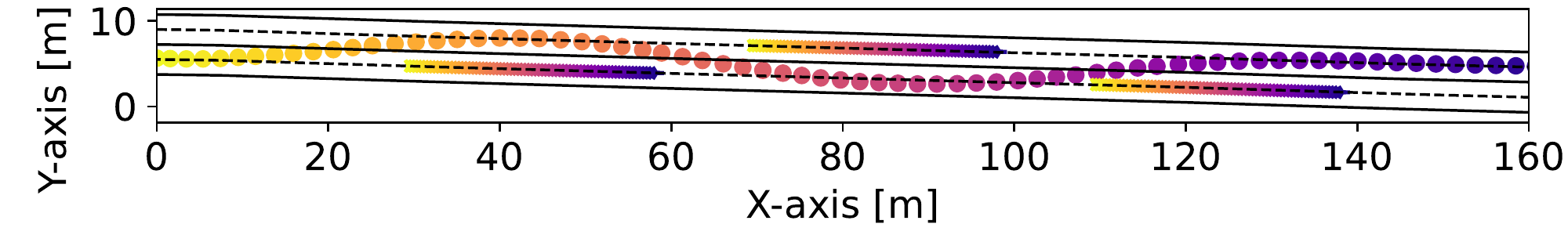}}
\\
\subfloat[]{\label{fig:tra_5}\includegraphics[width=0.9\linewidth]{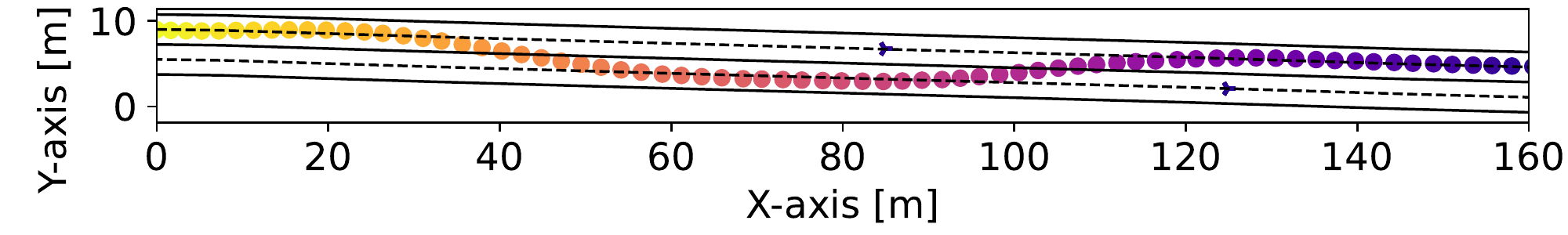}}
\caption{Trajectories under different traffic conditions.}
\label{f: experiment_tra}
\end{figure}

In addition, to increase the difficulty and test the robustness of the learned policy to external disturbances, we manually turn the steering wheel to impose interference during driving. The gray shaded area in Fig. \ref{f: experiment_example_2} represents the time interval during which we apply artificial steering wheel disturbances. The first interference occurred after driving about 10m. At this time, we manually turned the steering wheel to about -50$^\circ$ to make the ego car approach the right boundary of the road (See Fig. \ref{fig:experi_picture1}). When the IPC took over, the learned policy quickly sent out a left-turn command to make the vehicle quickly return to the right position (See Fig. \ref{fig:experi_picture2}). The second and fourth disturbances are similar to the first, during which we urgently turned the steering wheel to the left. The deviation caused by the two was also well corrected by the learned policy. Before the third interference occurred, the vehicle was in the initial stage of left lane change, aiming to avoid the preceding car. Then, we turned the steering wheel right to about -35$^\circ$ to interfere with the lane change (See Fig. \ref{fig:experi_picture3}). The learned policy still successfully completed the left lane change after taking over.
\begin{figure}[thpb]
\centering
\captionsetup{singlelinecheck = false,labelsep=period, font=small}
\captionsetup[subfigure]{justification=centering}
\subfloat[Trajectory]{\label{fig:experi_tra_d}\includegraphics[width=0.9\linewidth]{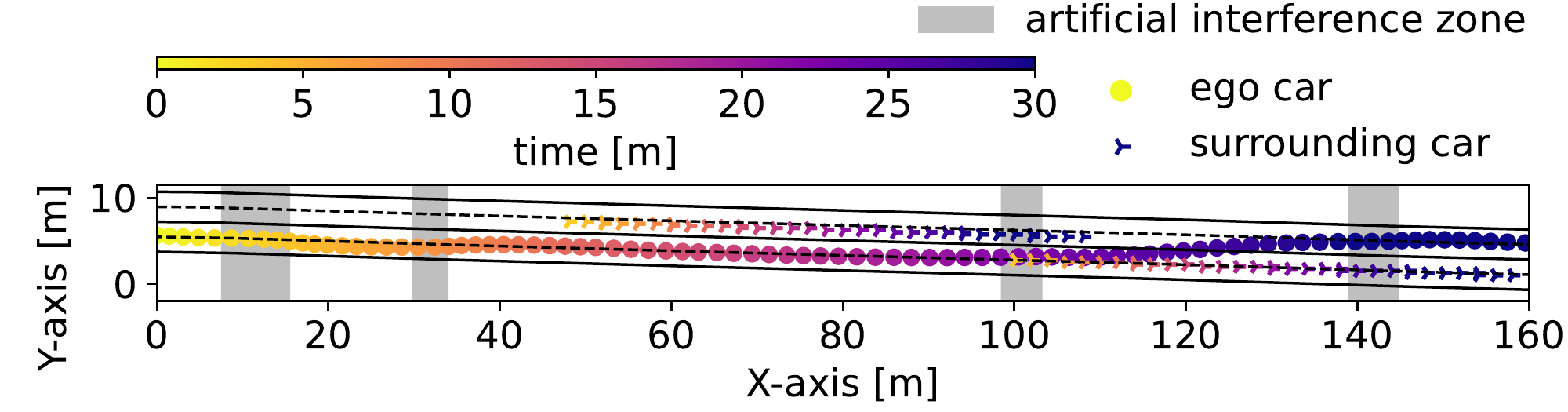}}
\\
\subfloat[Steering wheel angle]{\label{fig:experi_control_d}\includegraphics[width=0.9\linewidth]{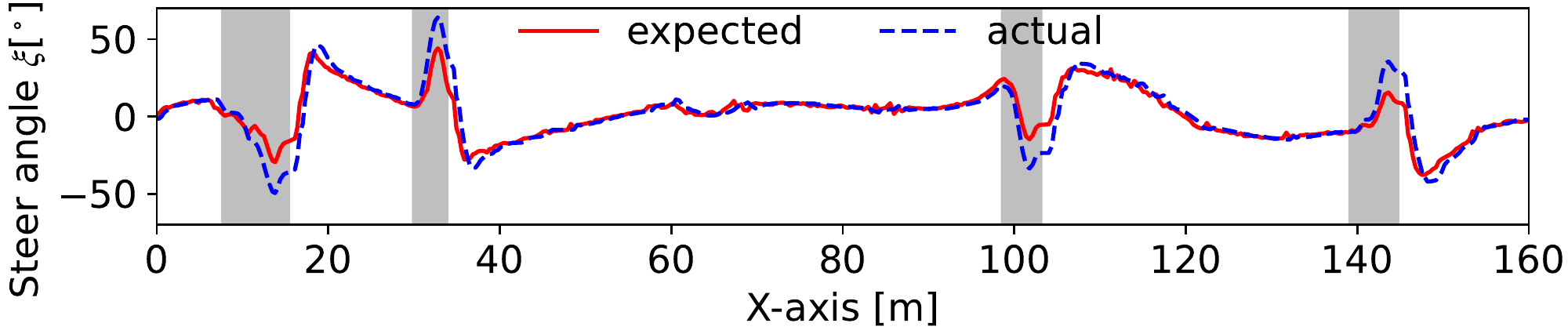}}
\\
\subfloat[Heading angle and yaw rate]{\label{fig:experi_state_d}\includegraphics[width=0.9\linewidth]{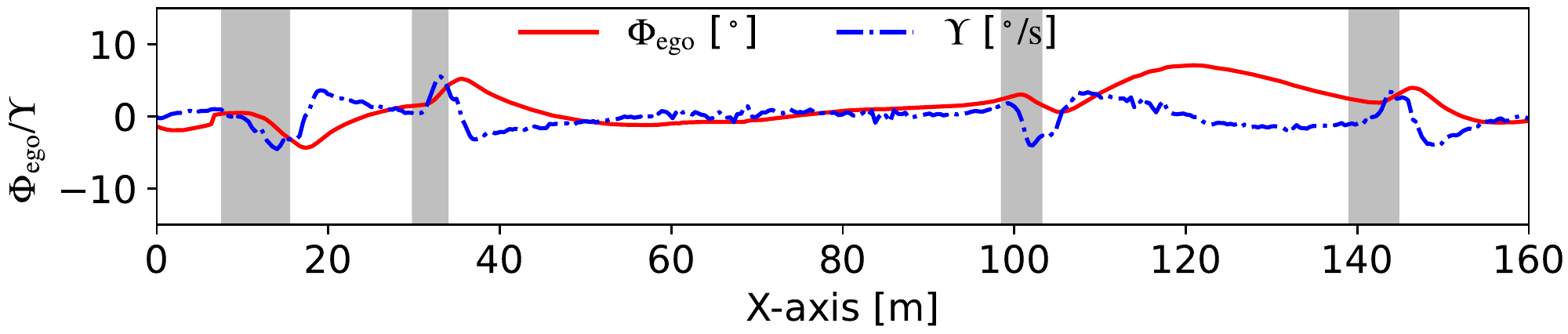}}
\\
\subfloat[Distance to lane center]{\label{fig:experi_deviation_d}\includegraphics[width=0.9\linewidth]{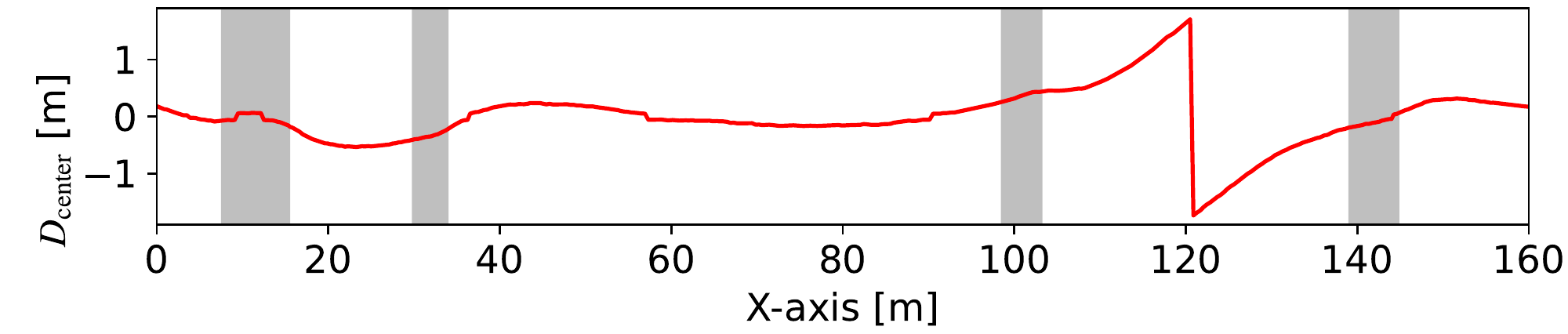}}
\caption{Experiment 2: trajectory, state and action curves. The gray shaded area represents the time interval during which we apply artificial steering wheel disturbances.}
\label{f: experiment_example_2}
\end{figure}

\begin{figure}[thpb]
\centering
\captionsetup{singlelinecheck = false,labelsep=period, font=small}
\captionsetup[subfigure]{justification=centering}
\subfloat[The 1st interference]{\label{fig:experi_picture1}\includegraphics[width=0.48\linewidth]{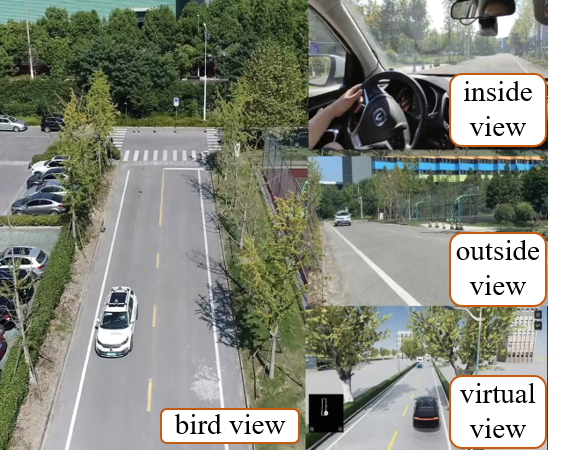}}\quad
\subfloat[The 1st take over ]{\label{fig:experi_picture2}\includegraphics[width=0.48\linewidth]{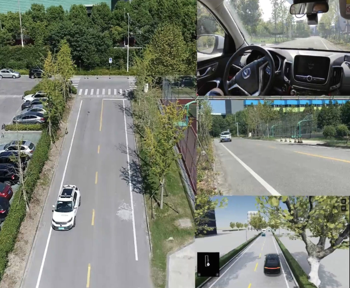}}
\\
\subfloat[The 3rd interference]{\label{fig:experi_picture3}\includegraphics[width=0.48\linewidth]{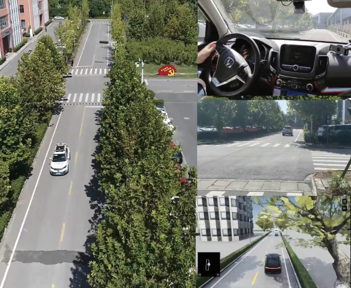}} \quad
\subfloat[The 3rd take over]{\label{fig:experi_picture4}\includegraphics[width=0.48\linewidth]{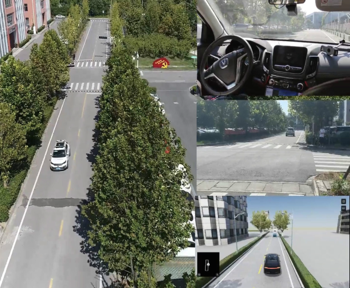}}
\caption{Images of experiment 2.}
\label{f: experiment_picture}
\end{figure}

To conclude, the learned policy of E-DSAC has the potential for real-vehicle applications. It can realize relatively safe and smooth decision-making and control under different traffic situations in multi-lane scenarios. Besides, when the steering wheel is disturbed by human beings, the learned policy can restore the vehicle to the proper status immediately after taking over. This implies that the learned policy is also compatible with human-machine shared driving.

\section{Conclusion}
\label{sec:conclusion}
In this paper, we propose an encoding distributional soft actor-critic (E-DSAC) for self-driving decision-making, which can deal with permutation sensitivity problems faced by existing related studies. Firstly, we develop an encoding distributional policy iteration (DPI) framework by embedding the encoding sum and concatenation method in the distributional RL framework. Then, the proposed DPI framework is proved to exhibit important properties in terms of convergence and global optimality. Based on the encoding DPI framework, we propose the E-DSAC algorithm by adding the gradient-based update rule of the feature NN to the policy evaluation process of the DSAC algorithm. We design a simulated multi-lane highway driving task and the corresponding reward function to verify the effectiveness of E-DSAC. Results show that the policy learned by E-DSAC can realize efficient, smooth, and relatively  safe  autonomous driving in the designed scenario. Compared with DSAC, E-DSAC has improved the final policy return by about three times, relaxing the requirement for predetermined sorting rules and vehicle number restrictions. Finally, a real vehicle test is conducted, and experimental
results show that the learned policy can smoothly complete maneuvers such as lane-keeping and
lane-changing in practical applications, so as to realize autonomous driving in response to different
surrounding vehicles. Besides, the learned policy has high robustness to  artificial steering wheel interference. This study poses great potential for the application of RL in actual driving scenarios.

\appendices
\section{Supplementary Results}
\subsection{Simulation Example}
\label{appen.example}
Fig. \ref{f:e2_trajectory} gives an instance of the overtaking process, and Fig. \ref{f:e2_state} shows the corresponding state curves. As shown in Fig. \ref{fig:e2_tra1} and \ref{fig:e2_tra2}, the ego vehicle first changes lanes to the innermost lane due to the too close following distance. After that, the self-vehicle found a suitable lane change position by accelerating and then completed the overtaking process through right lane change (See Fig. \ref{fig:e2_tra3}, \ref{fig:e2_tra4} and \ref{fig:e2_speed}) .
\begin{figure}[!htb]
\centering
\captionsetup{singlelinecheck = false,labelsep=period, font=small}
\captionsetup[subfigure]{justification=centering}
\subfloat[Initialization]{\label{fig:e2_tra1}\includegraphics[width=0.99\linewidth]{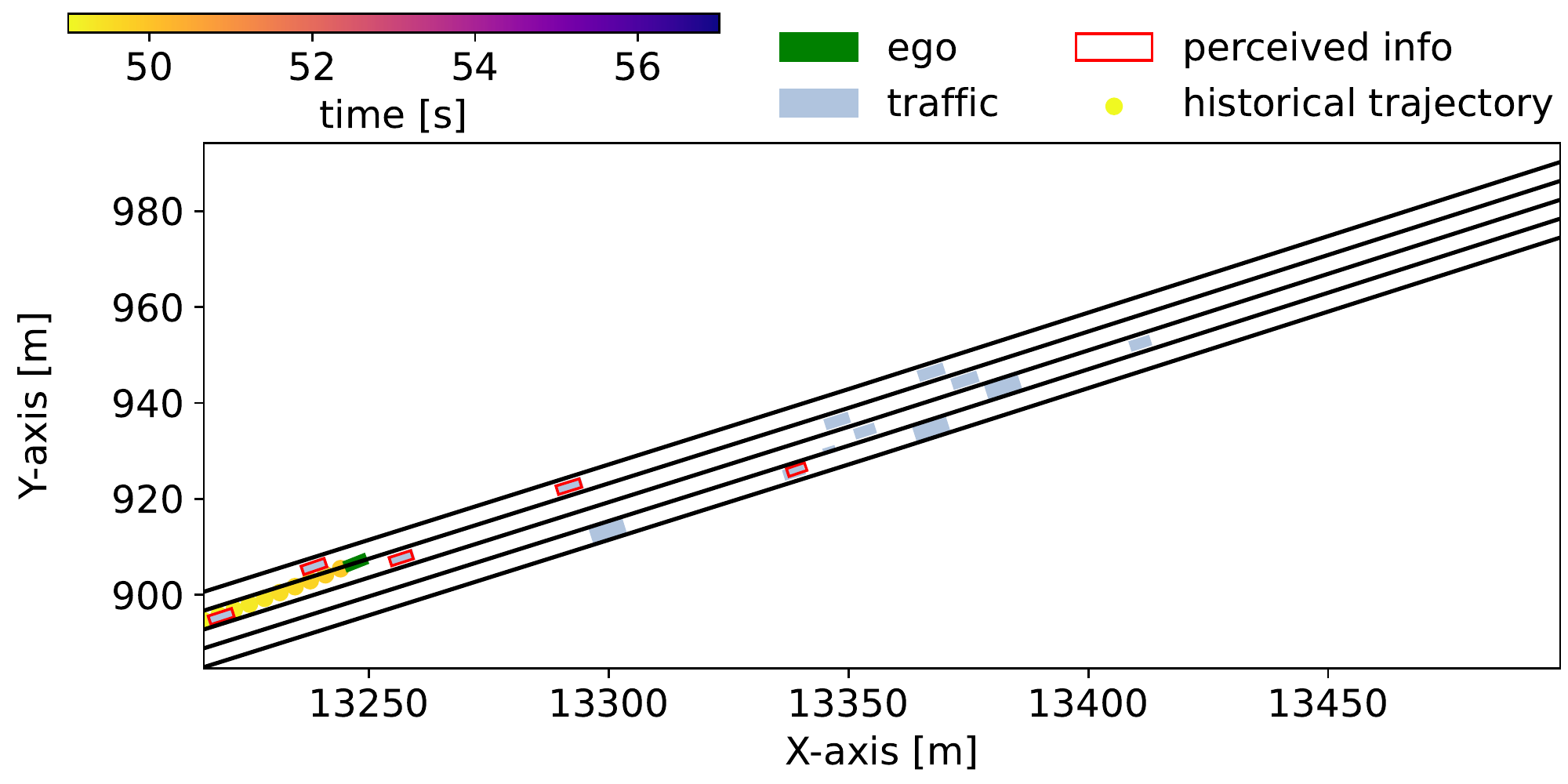}}
\\
\subfloat[Left lane change]{\label{fig:e2_tra2}\includegraphics[width=0.99\linewidth]{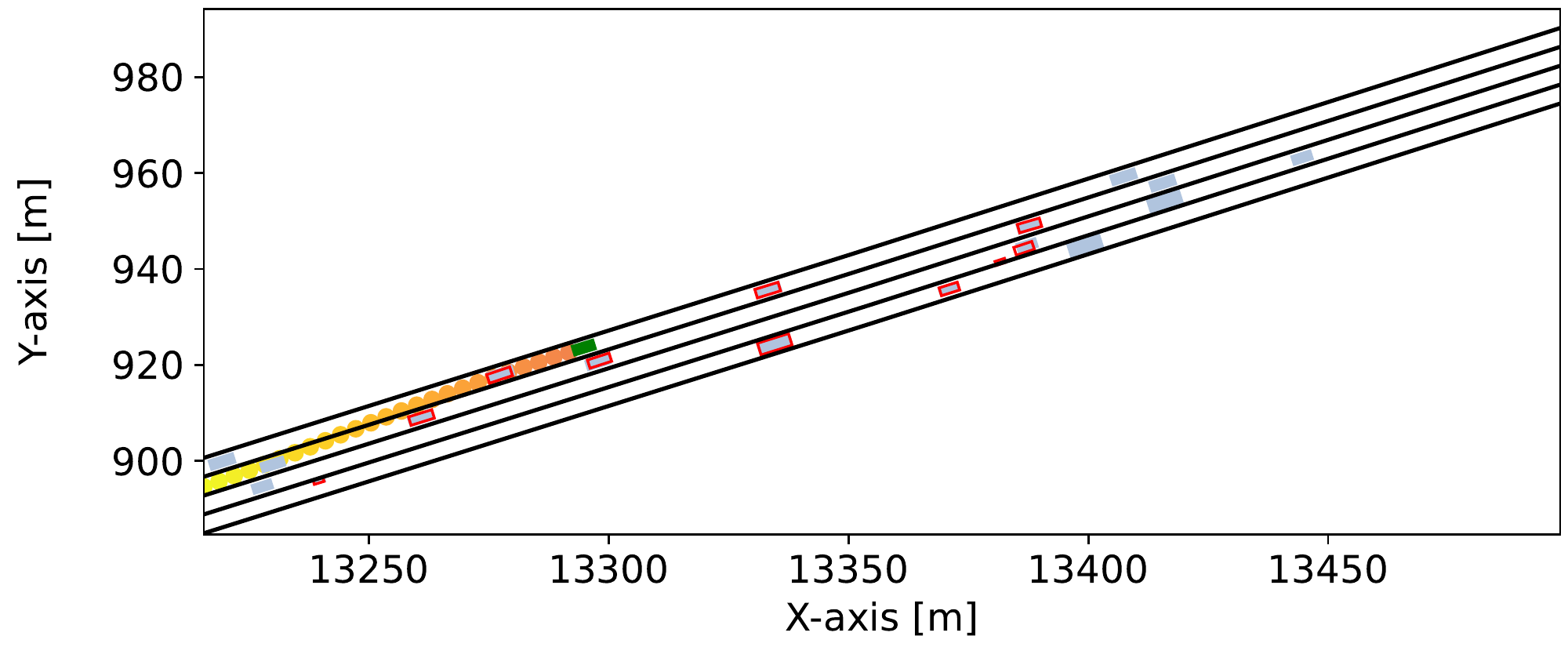}}
\\
\subfloat[Right lane change]{\label{fig:e2_tra3}\includegraphics[width=0.99\linewidth]{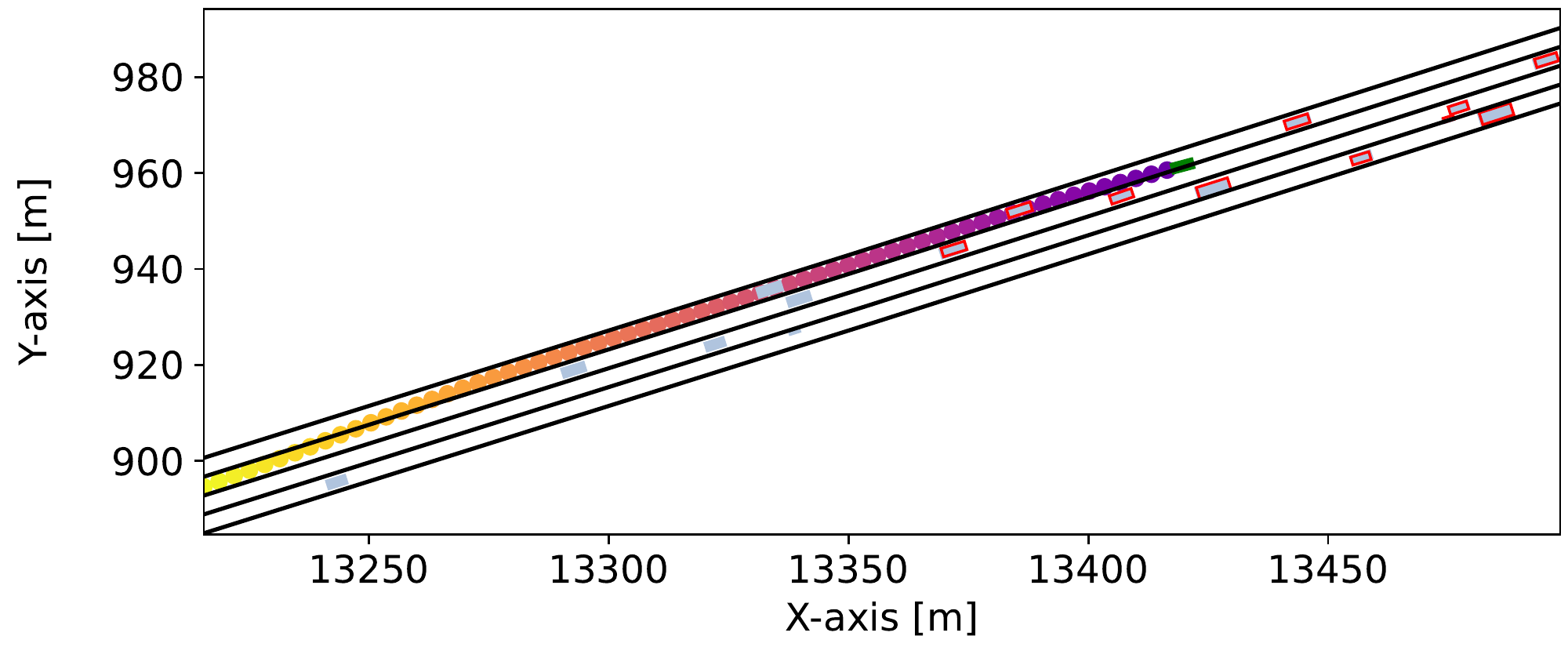}}
\\
\subfloat[Going straight]{\label{fig:e2_tra4}\includegraphics[width=0.99\linewidth]{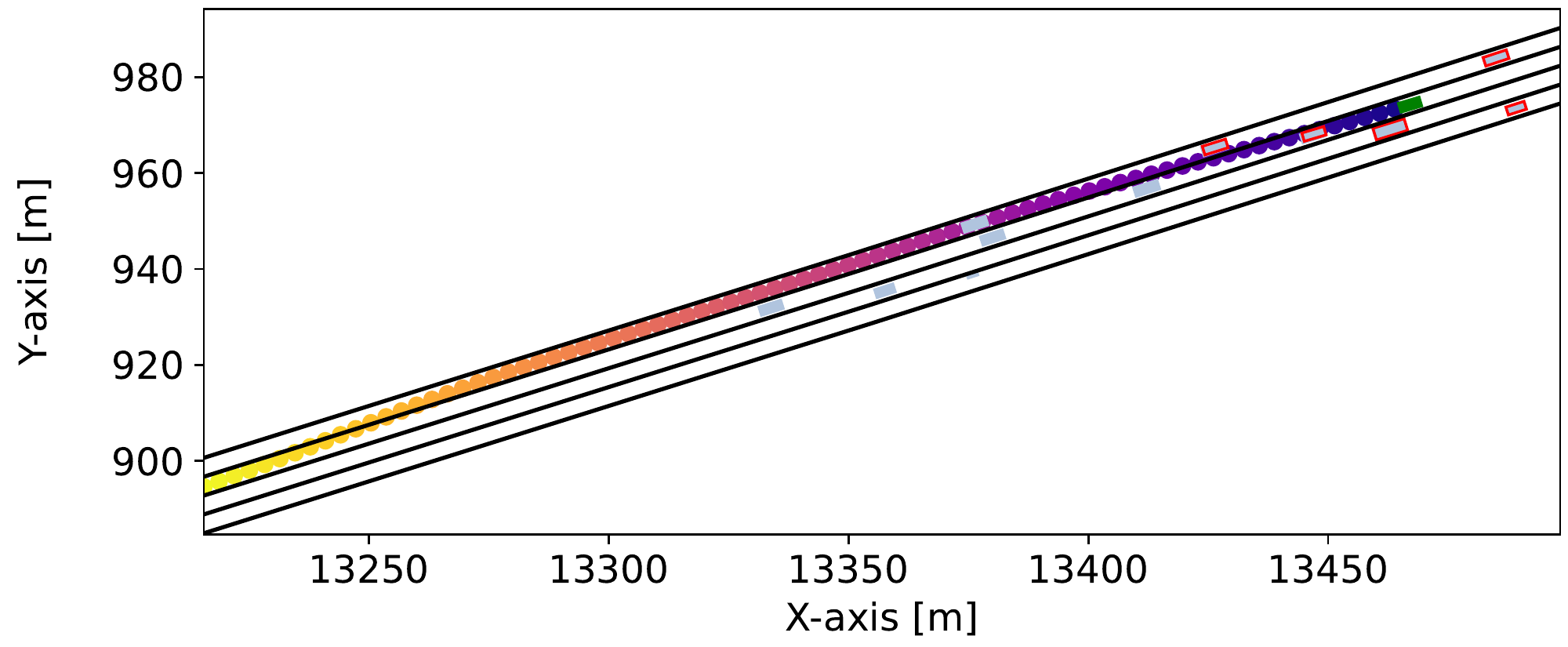}}\\
\caption{Simulation 2: trajectories. The red box represents the indicators of the perceived vehicle.}
\label{f:e2_trajectory}
\end{figure} 

\begin{figure}[!htb]
\centering
\captionsetup[subfigure]{justification=centering}
\subfloat[Action commands]{\label{fig:e2_action}\includegraphics[width=0.99\linewidth]{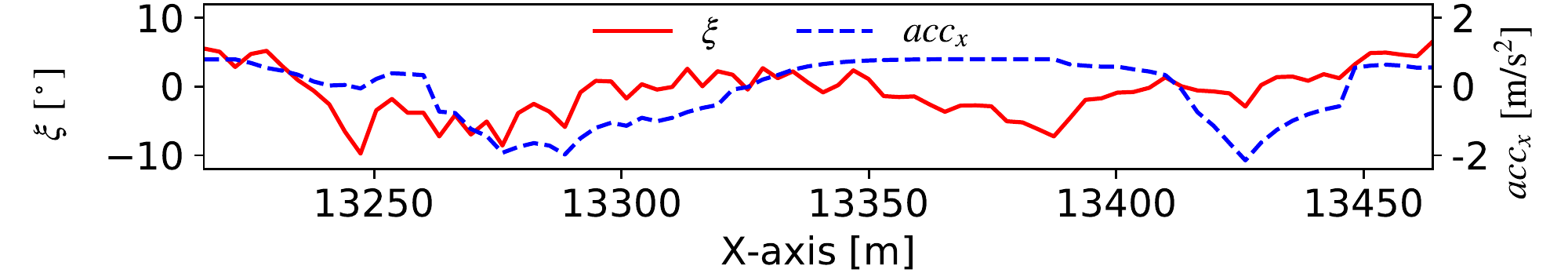}}\\
\subfloat[Heading angle and yaw rate]{\label{fig:e2_angle}\includegraphics[width=0.99\linewidth]{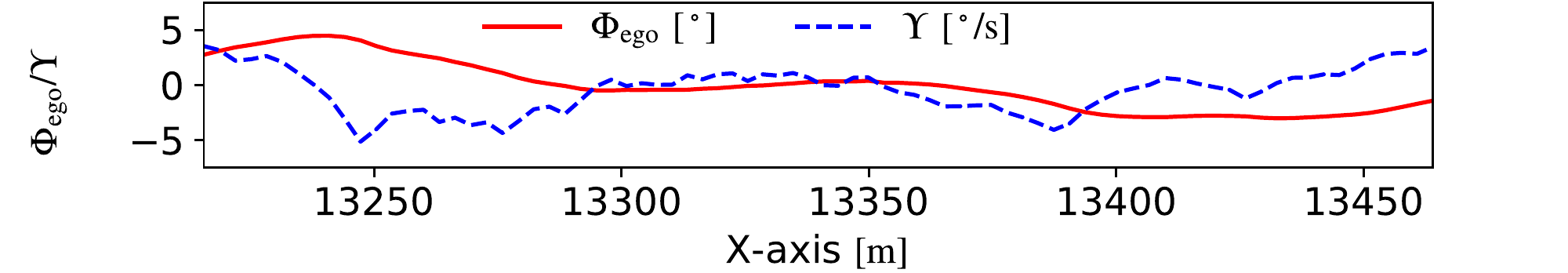}}\\
\subfloat[Speed]{\label{fig:e2_speed}\includegraphics[width=0.99\linewidth]{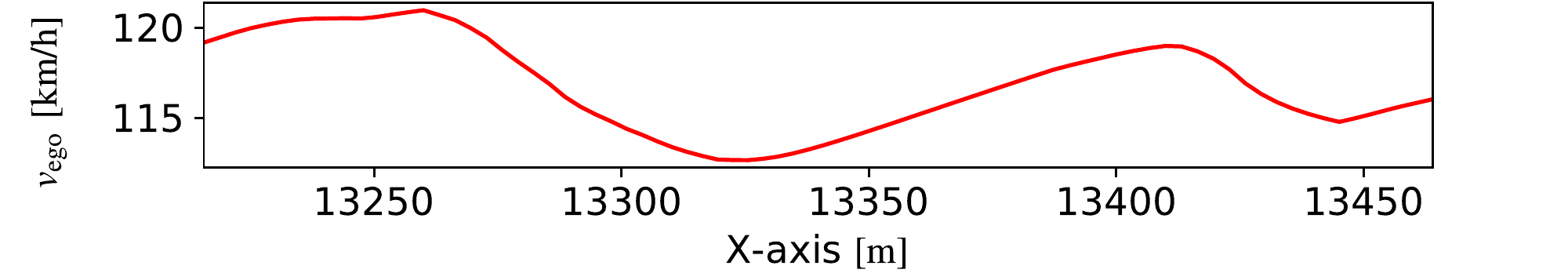}}\\
\caption{Simulation 2: state and control curves}
\label{f:e2_state}
\end{figure} 

\subsection{Extension}
\label{appen.extension}
The encoding DPI framework can also be extended to other RL algorithms, such as SAC, TD3, and DDPG. Based on similar ideas, we develop encoding versions of these algorithms, called E-SAC, E-TD3 and E-DDPG, respectively. We run these algorithms in the designed multi-lane highway driving task, and the learning curves are shown in Fig. \ref{f:return_extension}. The results show that the encoding RL framework has good compatibility with mainstream RL algorithms, and is of great significance for improving the policy performance in the field of RL-based autonomous driving.
\begin{figure}[!htb]
\captionsetup{singlelinecheck = false,labelsep=period, font=small}
\centering{\includegraphics[width=0.4\textwidth]{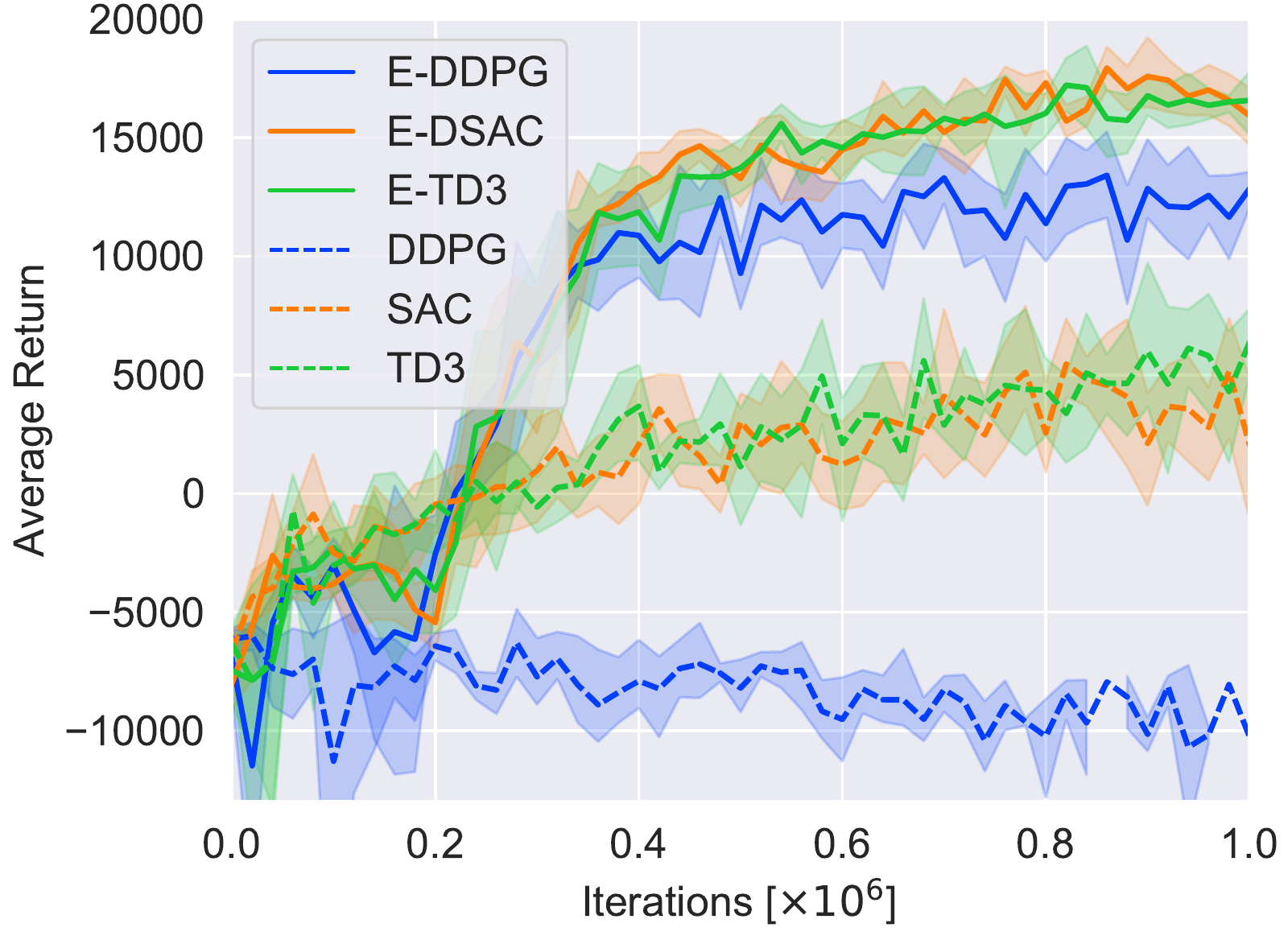}}
\caption{Return comparison. The solid lines correspond to the mean and the shaded regions correspond to 95\% confidence interval over 5 runs.}
\label{f:return_extension}
\end{figure}

% \appendices
% \section{Proof of Convergence}
% \label{appen.proof}

% \section*{Acknowledgment}
% We would like to acknowledge *** for his valuable suggestions.

\ifCLASSOPTIONcaptionsoff
  \newpage
\fi

\bibliographystyle{ieeetr}
\bibliography{reference}

\begin{thebibliography}{10}

\bibitem{pendleton2017perception}
S.~D. Pendleton, H.~Andersen, X.~Du, X.~Shen, M.~Meghjani, Y.~H. Eng, D.~Rus,
  and M.~H. Ang, ``Perception, planning, control, and coordination for
  autonomous vehicles,'' {\em Machines}, vol.~5, no.~1, p.~6, 2017.

\bibitem{buehler2009darpa}
M.~Buehler, K.~Iagnemma, and S.~Singh, {\em The DARPA urban challenge:
  autonomous vehicles in city traffic}, vol.~56.
\newblock springer, 2009.

\bibitem{mnih2015DQN}
V.~Mnih, K.~Kavukcuoglu, D.~Silver, A.~A. Rusu, J.~Veness, M.~G. Bellemare,
  A.~Graves, M.~Riedmiller, A.~K. Fidjeland, G.~Ostrovski, {\em et~al.},
  ``Human-level control through deep reinforcement learning,'' {\em Nature},
  vol.~518, no.~7540, p.~529, 2015.

\bibitem{silver2017mastering}
D.~Silver, J.~Schrittwieser, K.~Simonyan, I.~Antonoglou, A.~Huang, A.~Guez,
  T.~Hubert, L.~Baker, M.~Lai, A.~Bolton, {\em et~al.}, ``Mastering the game of
  go without human knowledge,'' {\em nature}, vol.~550, no.~7676, pp.~354--359,
  2017.

\bibitem{lillicrap2015DDPG}
T.~P. Lillicrap, J.~J. Hunt, A.~Pritzel, N.~Heess, T.~Erez, Y.~Tassa,
  D.~Silver, and D.~Wierstra, ``Continuous control with deep reinforcement
  learning,'' in {\em 4th International Conference on Learning Representations
  (ICLR 2016)}, (San Juan, Puerto Rico), 2016.

\bibitem{wayve}
A.~{Kendall}, J.~{Hawke}, D.~{Janz}, P.~{Mazur}, D.~{Reda}, J.~{Allen},
  V.~{Lam}, A.~{Bewley}, and A.~{Shah}, ``Learning to drive in a day,'' in {\em
  2019 International Conference on Robotics and Automation (ICRA)},
  pp.~8248--8254, 2019.

\bibitem{wolf2017learning}
P.~Wolf, C.~Hubschneider, M.~Weber, A.~Bauer, J.~H{\"a}rtl, F.~D{\"u}rr, and
  J.~M. Z{\"o}llner, ``Learning how to drive in a real world simulation with
  deep q-networks,'' in {\em Intelligent Vehicles Symposium (IV)}, (Los
  Angeles, California), pp.~244--250, IEEE, 2017.

\bibitem{perot2017end}
E.~Perot, M.~Jaritz, M.~Toromanoff, and R.~De~Charette, ``End-to-end driving in
  a realistic racing game with deep reinforcement learning,'' in {\em
  Proceedings of the IEEE Conference on Computer Vision and Pattern Recognition
  Workshops}, (Columbus, Ohio), pp.~3--4, IEEE, 2017.

\bibitem{jaritz2018end}
M.~Jaritz, R.~De~Charette, M.~Toromanoff, E.~Perot, and F.~Nashashibi,
  ``End-to-end race driving with deep reinforcement learning,'' in {\em
  International Conference on Robotics and Automation (ICRA)}, pp.~2070--2075,
  IEEE, 2018.

\bibitem{zou2018inverse}
Q.~Zou, H.~Li, and R.~Zhang, ``Inverse reinforcement learning via neural
  network in driver behavior modeling,'' in {\em Intelligent Vehicles Symposium
  (IV)}, (Changshu, Suzhou), pp.~1245--1250, IEEE, 2018.

\bibitem{chen2019SACdriving}
J.~Chen, B.~Yuan, and M.~Tomizuka, ``Model-free deep reinforcement learning for
  urban autonomous driving,'' in {\em 22nd International Conference on
  Intelligent Transportation Systems (ITSC)}, (Auckland, New Zealand),
  pp.~2765--2771, IEEE, 2019.

\bibitem{isele2018navigating}
D.~Isele, R.~Rahimi, A.~Cosgun, K.~Subramanian, and K.~Fujimura, ``Navigating
  occluded intersections with autonomous vehicles using deep reinforcement
  learning,'' in {\em International Conference on Robotics and Automation
  (ICRA)}, pp.~2034--2039, IEEE, 2018.

\bibitem{wang2017formulation}
P.~Wang and C.-Y. Chan, ``Formulation of deep reinforcement learning
  architecture toward autonomous driving for on-ramp merge,'' in {\em 20th
  International Conference on Intelligent Transportation Systems (ITSC)},
  (Yokohama, Japan), pp.~1--6, IEEE, 2017.

\bibitem{duan2020hierarchical}
J.~Duan, S.~E. Li, Y.~Guan, Q.~Sun, and B.~Cheng, ``Hierarchical reinforcement
  learning for self-driving decision-making without reliance on labelled
  driving data,'' {\em IET Intelligent Transport Systems}, vol.~14, no.~5,
  pp.~297--305, 2020.

\bibitem{GUAN}
Y.~{Guan}, Y.~{Ren}, S.~E. {Li}, Q.~{Sun}, L.~{Luo}, and K.~{Li}, ``Centralized
  cooperation for connected and automated vehicles at intersections by proximal
  policy optimization,'' {\em IEEE Transactions on Vehicular Technology},
  vol.~69, no.~11, pp.~12597--12608, 2020.

\bibitem{schulman2017PPO}
J.~Schulman, F.~Wolski, P.~Dhariwal, A.~Radford, and O.~Klimov, ``Proximal
  policy optimization algorithms,'' {\em arXiv preprint arXiv:1707.06347},
  2017.

\bibitem{mirchevska2018high}
B.~Mirchevska, C.~Pek, M.~Werling, M.~Althoff, and J.~Boedecker, ``High-level
  decision making for safe and reasonable autonomous lane changing using
  reinforcement learning,'' in {\em 21st International Conference on
  Intelligent Transportation Systems (ITSC)}, (Maui, Hawaii), pp.~2156--2162,
  IEEE, 2018.

\bibitem{wang2018reinforcement}
P.~Wang, C.-Y. Chan, and A.~de~La~Fortelle, ``A reinforcement learning based
  approach for automated lane change maneuvers,'' in {\em Intelligent Vehicles
  Symposium (IV)}, (Changshu, Suzhou), pp.~1379--1384, IEEE, 2018.

\bibitem{wang2019continuous}
P.~Wang, H.~Li, and C.-Y. Chan, ``Continuous control for automated lane change
  behavior based on deep deterministic policy gradient algorithm,'' in {\em
  Intelligent Vehicles Symposium (IV)}, pp.~1454--1460, IEEE, 2019.

\bibitem{dosovitskiy2017carla}
A.~Dosovitskiy, G.~Ros, F.~Codevilla, A.~Lopez, and V.~Koltun, ``Carla: An open
  urban driving simulator,'' in {\em Conference on robot learning}, pp.~1--16,
  PMLR, 2017.

\bibitem{SUMO2018}
P.~Lopez, M.~Behrisch, L.~Bieker-Walz, J.~Erdmann, Y.-P. Fl{\"o}tter{\"o}d,
  R.~Hilbrich, L.~L{\"u}cken, and Johannes, ``Microscopic traffic simulation
  using sumo,'' in {\em International Conference on Intelligent Transportation
  Systems (ITSC)}, IEEE, 2018.

\bibitem{duan2021fixeddimensional}
J.~Duan, D.~Yu, S.~E. Li, W.~Wang, Y.~Ren, Z.~Lin, and B.~Cheng,
  ``Fixed-dimensional and permutation invariant state representation of
  autonomous driving,'' 2021.

\bibitem{duan2021distributional}
J.~Duan, Y.~Guan, S.~E. Li, Y.~Ren, Q.~Sun, and B.~Cheng, ``Distributional soft
  actor-critic: Off-policy reinforcement learning for addressing value
  estimation errors,'' {\em IEEE Transactions on Neural Networks and Learning
  Systems}, 2021.

\bibitem{Haarnoja2017Soft-Q}
T.~Haarnoja, H.~Tang, P.~Abbeel, and S.~Levine, ``Reinforcement learning with
  deep energy-based policies,'' in {\em Proceedings of the 34th International
  Conference on Machine Learning, (ICML 2017)}, (Sydney, NSW, Australia),
  pp.~1352--1361, 2017.

\bibitem{Haarnoja2018SAC}
T.~Haarnoja, A.~Zhou, P.~Abbeel, and S.~Levine, ``Soft actor-critic: Off-policy
  maximum entropy deep reinforcement learning with a stochastic actor,'' in
  {\em Proceedings of the 35th International Conference on Machine Learning
  (ICML 2018)}, (Stockholmsmässan, Stockholm Sweden), pp.~1861--1870, PMLR,
  2018.

\bibitem{Hornik1990Universal}
K.~Hornik, M.~Stinchcombe, and H.~White, ``Universal approximation of an
  unknown mapping and its derivatives using multilayer feedforward networks,''
  {\em Neural Networks}, vol.~3, no.~5, pp.~551--560, 1990.

\bibitem{Haarnoja2018ASAC}
T.~Haarnoja, A.~Zhou, K.~Hartikainen, G.~Tucker, S.~Ha, J.~Tan, V.~Kumar,
  H.~Zhu, A.~Gupta, P.~Abbeel, {\em et~al.}, ``Soft actor-critic algorithms and
  applications,'' {\em arXiv preprint arXiv:1812.05905}, 2018.

\bibitem{cao2020novel}
M.~Cao, J.~Chen, and J.~Wang, ``A novel vehicle tracking method for cross-area
  sensor fusion with reinforcement learning based gmm,'' in {\em 2020 American
  Control Conference (ACC)}, pp.~442--447, IEEE, 2020.

\bibitem{hendrycks2016gelu}
D.~Hendrycks and K.~Gimpel, ``Gaussian error linear units (gelus),'' {\em arXiv
  preprint arXiv:1606.08415}, 2016.

\bibitem{Diederik2015Adam}
D.~P. Kingma and J.~Ba, ``Adam: {A} method for stochastic optimization,'' in
  {\em 3rd International Conference on Learning Representations, (ICLR 2015)},
  (San Diego, CA, USA), 2015.

\bibitem{Fujimoto2018TD3}
S.~Fujimoto, H.~van Hoof, and D.~Meger, ``Addressing function approximation
  error in actor-critic methods,'' in {\em Proceedings of the 35th
  International Conference on Machine Learning (ICML 2018)},
  (Stockholmsmässan, Stockholm Sweden), pp.~1587--1596, PMLR, 2018.

\end{thebibliography}

\end{document}